\documentclass{article}

\usepackage[margin=1in]{geometry}

\usepackage{authblk}
\usepackage{amsmath,amssymb,amsfonts,enumerate,enumitem}
\usepackage{amsthm}

\usepackage[utf8]{inputenc}
\usepackage[T1]{fontenc}
\usepackage{subcaption}
\usepackage[dvipsnames]{xcolor}
\usepackage{floatrow}
\usepackage{bm,bbm}
\usepackage{relsize}

\usepackage{url}
\usepackage{hyperref}
\usepackage{cleveref}

\usepackage{thmtools} 
\usepackage{thm-restate}

\usepackage{mathtools}
\usepackage{commath}

\usepackage{relsize}
\usepackage{array}
\usepackage{upgreek}

\usepackage{enumitem}

\usepackage{multicol}
\usepackage{multirow}

\usepackage{microtype}

\usepackage{algorithm,algcompatible}
\usepackage{algpseudocode}

\usepackage[leftcaption]{sidecap}
\sidecaptionvpos{figure}{c}

\usepackage{tikz}
\usepackage{pgfplots}
\pgfplotsset{compat=newest}
\usetikzlibrary{patterns}
\usetikzlibrary{plotmarks}
\usepgfplotslibrary{fillbetween}

\usetikzlibrary{arrows}
\usetikzlibrary{arrows.meta}
\usetikzlibrary{bending}

\usepackage{nicefrac}

\usepackage{wasysym}

\usepackage{subcaption}

\newcommand{\R}{\mathbb{R}}

\DeclareMathOperator*{\Expect}{\mathbb{E}}
\DeclareMathOperator*{\Var}{\mathbb{V}}

\DeclareMathOperator*{\Prob}{\mathbb{P}}

\newcommand{\M}{\mathcal{M}}

\newcommand{\TMix}{\tau_{\rm mix}}
\newcommand{\TRel}{\tau_{\rm rx}}

\newcommand{\EigTwo}{\lambda}

\newcommand{\EigTwoBound}{\Lambda}

\newcommand{\SVar}{v_{\pi}} 

\newcommand{\ITVn}[1]{\mathrm{trv}^{(#1)}}

\newcommand{\ITRelVar}{\ITVn{\tau \mathrm{rx}}}

\newcommand{\RelITRelVar}[2]{\mathrm{Reltrv}^{#2}_{#1}}

\newcommand{\frange}{R}

\newcommand{\NIterations}{I}

\newcommand{\cyrus}[1]{\textcolor{green!50!black}{\textsc{Cyrus:} \emph{#1}}}
\newcommand{\shahrzad}[1]{\textcolor{blue!60!black}{\textsc{Shahrzad:} \emph{#1}}}
\newcommand{\sophia}[1]{\textcolor{orange!50!black}{\textsc{sophia:} \emph{#1}}}
\newcommand{\todo}[1]{\textcolor{red!50!black}{\textsc{ToDo:} \emph{#1}}}
\newcommand{\eli}[1]{\textcolor{red!50!black}{\textsc{Eli:} \emph{#1}}}

\iftrue

\renewcommand{\cyrus}[1]{}
\renewcommand{\shahrzad}[1]{}
\renewcommand{\sophia}[1]{}
\renewcommand{\todo}[1]{}
\renewcommand{\eli}[1]{}

\fi

\makeatletter
\newcommand\longleftrightarrowfill@{%
  \arrowfill@\leftarrow\relbar\rightarrow}
\makeatother

\crefname{algorithm}{alg.}{algs.}
\Crefname{algorithm}{Algorithm}{Algorithms}
\crefname{appendix}{appx.}{appcs.}
\Crefname{appendix}{Appendix}{Appendices}
\crefname{corollary}{coro.}{coros.}
\Crefname{corollary}{Corollary}{Corollaries}
\crefname{conjecture}{conjecture}{conjectures}
\Crefname{conjecture}{Conjecture}{Conjectures}
\crefformat{conjecture}{conjecture~#2#1#3}
\crefmultiformat{conjecture}{Conjectures~#2#1#3}{\ and~#2#1#3}{, #2#1#3}{\ and~#2#1#3}
\crefname{definition}{def.}{defs.}
\Crefname{definition}{Definition}{Definition}
\crefname{figure}{fig.}{figs.}
\Crefname{figure}{Figure}{Figures}
\crefname{lemma}{lemma}{lemmas}
\Crefname{lemma}{Lemma}{Lemmas}
\crefname{proposition}{prop.}{props.}
\Crefname{proposition}{Proposition}{Propositions}
\Crefname{section}{Section}{Sections}
\crefname{section}{sect.}{sect.}
\crefname{subsection}{sect.}{sect.}
\Crefname{subsection}{Section}{Sections}
\crefname{subsubsection}{sect.}{sect.}
\Crefname{subsubsection}{Section}{Sections}
\crefname{table}{table}{tables}
\Crefname{table}{Table}{Tables}
\crefname{theorem}{thm.}{thms.}
\Crefname{theorem}{Theorem}{Theorems}

\Crefname{theorem}{Thm.}{Thms.}

\crefname{section}{\S}{\S}
\crefname{subsection}{\S}{\S}
\Crefname{subsection}{\S}{\S}
\crefname{subsubsection}{\S}{\S}
\Crefname{subsubsection}{\S}{\S}

\newtheorem{defin}{Definition}[section]
\newtheorem{theorem}{Theorem}[section]
\newtheorem*{theorem*}{Theorem}

\newtheorem{lemma}[theorem]{Lemma}
\newtheorem{coro}[theorem]{Corollary}

\newtheorem*{problemn}{Problem}
\newtheorem{problem}{Problem}

\newcommand{\paralgo}{\scalebox{0.92}[0.98]{\textsc{ParallelTraceGibbs}}}
\newcommand{\superalgo}{\scalebox{0.92}[0.98]{\textsc{SuperChainTraceGibbs}}}
\newcommand{\proc}{\scalebox{0.92}[0.98]{\textsc{RelMeanEst}}}
\newcommand{\tpa}{\scalebox{0.92}[0.98]{\textsc{Tpa}}}
\newcommand{\TPA}{\tpa}

\newcommand{\ceil}[1]{\left\lceil #1 \right\rceil}

\newcommand{\Zi}[1]{Z{({#1})}}
\newcommand{\gibbs}[1]{\pi_{#1}}
\newcommand{\gibbschain}[2]{{\mathcal G}_{#1,#2}}

\newcommand{\LandauTheta}{\Theta}

\usepackage[T1]{fontenc}
\usepackage{lmodern}

\usepackage[utf8]{inputenc} 
\usepackage[T1]{fontenc}    
\usepackage{hyperref}       
\usepackage{url}            
\usepackage{booktabs}       
\usepackage{amsfonts}       
\usepackage{nicefrac}       
\usepackage{microtype}      
\usepackage{xcolor}         

\title{
Fast Doubly-Adaptive MCMC 
to 
Estimate 
the Gibbs Partition Function 
with Weak Mixing Time Bounds

}

\author[1]{Shahrzad Haddadan\textcircled{r}}
\author[2]{Yue Zhuang\textcircled{r}}
\author[3]{Cyrus Cousins\textcircled{r}}
\author[4]{Eli Upfal}
\affil[3,4]{  Department of Computer Science, Brown University }
\affil[1,2]{  Data Science Initiative, Brown University }
\affil[1]{
\footnotesize{\texttt{shahrzad\_haddadan@brown.edu}}}
{\affil[2]{
\footnotesize{\texttt{yue\_zhuang1@brown.edu}}}}
{\affil[3]{
\footnotesize{\texttt{cyrus\_cousins@brown.edu}}}
\affil[4]{
\footnotesize{\texttt{eli\_upfal@brown.edu}}}}

\date{}
\begin{document}
\maketitle

\begin{abstract}

We present a novel method for reducing the computational complexity of rigorously estimating the \emph{partition functions} (normalizing constants) of Gibbs (Boltzmann) distributions,
which arise ubiquitously in probabilistic graphical models.
A major obstacle to practical applications of Gibbs distributions is the need to estimate their {partition functions}.
The state of the art in addressing this problem is multi-stage algorithms, which consist of a cooling schedule, and a mean estimator in each step of the schedule. 
While the cooling schedule in these algorithms is adaptive, the mean estimation computations use MCMC as a black-box to draw approximate samples.
We develop a \emph{doubly adaptive} approach, combining the adaptive cooling schedule with an adaptive MCMC mean estimator, whose number of Markov chain steps adapts dynamically to the underlying chain.
Through rigorous theoretical analysis, we prove that our method outperforms the state of the art algorithms in several factors: (1) The computational complexity of our method is smaller; (2) Our method is less sensitive to loose bounds on mixing times, an inherent component in these algorithms; and (3) The improvement obtained by our method is particularly significant in the most challenging regime of high-precision estimation.
We demonstrate the advantage of our method in experiments run on classic factor graphs, such as voting models and Ising models. 

\end{abstract}

\section{Introduction}\label{sec:1}

The Gibbs (Boltzmann) distribution is a family of probability distributions of exponential form. First introduced in the context of statistical mechanics~\cite{Gibbs1902}, Gibbs distributions are now ubiquitous in a variety of other disciplines, such as chemistry~\cite{hellweg2017chm,gibbs1976chm}, economics~\cite{econ2,econ2012} and machine learning. Gibbs distributions are typically used to model the global state of a system as a function of a collection of interdependent random variables, each representing local states in the system. The dependencies in the system are modeled by a \emph{Hamiltonian} function, and the probability distribution is inversely proportional to exponent of the Hamiltonian scaled by the \emph{temperature} (see \cref{eq:gibbsdist} \cref{sc:prelim}). 

Gibbs distributions provide potent statistical inference tools in many machine learning applications. They appear in probabilistic graphical models~\cite{GraphicalmodelsBook,Factorie,libDAI}, including restricted Boltzmann machines \cite{Tosh2016MixingRF,RBMKrause2019}, Markov random fields \cite{kemeny2012denumerable,MRFref2014}, and Bayes networks  \cite{BayesianNet}, and are applied in the analysis of images and graphical data \cite{image,imageGibbs2017,imageGibbs,imageGibbs1984},  topic modeling (LDA) \cite{Griffiths02gibbssamplingLDA,LDA2008,NIPS2007_LDA,topicmodel},
and more~\cite{Cheng2015Gibbs,Plummer2003JAGSAP,Afshar2016ClosedFormGS,DeSaICML2016,GibbsPLMRdesa18a,Desa2019Gibbs,Gibbsredpmlr-v15-gonzalez11a,PGM2016scan,Openbugs,prasad2020learning}.

A major obstacle in applying the Gibbs distribution in practice is the need to compute, or estimate, its  \emph{partition function} (normalizing constant), henceforth written GPF. 
The partition function
 is defined over the Cartesian product of supports of a (typically large) number of variables, making exact computation intractable. Furthermore, due to interdependence of variables, exact sampling is not practically feasible, thus
Markov-chain Monte-Carlo (MCMC) solutions for this problem have been extensively studied ~\cite{neal2001annealed, countingsamplingreduction1986jurrumvaziranivalient,Fishman1994ChoosingSP,Stefankovich2009simulateannealing,gibbsvigoda2007,hubergibbs,Bezkov2006AcceleratingSA,Kolmogorovgibbs,KolmoHarrisGibbs, karagiannis2013annealed}.

Like other MCMC methods, here various heuristics are used. The most well-known  heuristics are the \emph{annealed importance sampling} \cite{neal2001annealed,stordal2015iterative,karagiannis2013annealed}  or \emph{convergence diagnostics methods}  \cite{diag1,diag3,diag4,diag5}. 
Unfortunately, these methods are often error-prone, as their correctness is only proven asymptotically, without rigorous mathematical analysis to bound their estimation error with \emph{finite samples}. In fact, theoretical findings have shown that with no prior knowledge of relevant measures, such as the variance of importance weights in annealed importance sampling, or upper bounds on mixing or relaxation times for diagnostic methods, these methods are either unreliable or computationally intractable 
(see \cite[section~4]{neal2001annealed} or \cite{bhatnagar2011computational,trelestHsuPeres2019}).

On the other hand, theoreticians study this problem by designing Fully Polynomial Randomized Approximation Schemes (FPRAS) (see \cref{prob:fpras}).  
The state of the art FPRAS for estimating the GPF is a multi-stage algorithm involving a sequence of functions at various temperatures, such that the expectation of the product of these functions, or the product of the expectations of said functions, is the GPF.  
FPRAS's are proven to produce (approximate) 
solutions w.h.p., but their performance guarantees  rely  on available upper-bounds on various measures such as variances of estimators or mixing times of Markov chains. In static algorithms, these upper-bounds are given \emph{a priori}, and adaptive\footnote{The usage of the word ``adaptive'' here refers to algorithms which
draw samples progressively and 
adapt their sample complexity  based on empirical estimates until desired conditions are met, as it has been used in \cite{hubergibbs,Kolmogorovgibbs} (see \cref{sc:prelim}), and should not confused with the work of \cite{roberts2009examples}.} algorithms estimate them dynamically, while increasing the sample size until desired properties are mathematically guaranteed.
Thus, adaptive algorithms are less sensitive to looseness of known upper-bounds, more robust, often faster, and more easily applied to various settings.


Most of the research on designing FPRAS's for the GPF is focused on designing   adaptive algorithms to produce  sequences (\emph{cooling schedules}) with minimum  length   while keeping the {variances} of estimators small (thus removing the need to have a-priori known bounds on variances).   
In contrast, the computation of the sequence of mean estimates, which dominates the total computation cost, 
is done by black-box MCMC estimators, with \emph{a priori} known upper bounds on the mixing times of the chains. These upper bounds are
 often loose, and improving them for particular models
is a challenging active area of research \cite{anari2020spectral,gibbsvigoda2007,GibbsVigoda2021,SodaColoring,matchingsanari,haddadan2019mixing}. 
In order to complement the adaptive cooling schedule and reduce dependence on \emph{a priori} bounds on Markov chains' mixing times, it seems necessary to  design an \emph{adaptive}  procedure with theoretical guarantees for MCMC-mean estimation.

In this work we develop a \emph{doubly adaptive} FPRAS, combining the adaptive cooling schedule with adaptive MCMC mean estimator that dynamically adapts the number of Markov chain steps to the observed underlying chain.
Through rigorous theoretical analysis, we prove that our method outperforms the state of the art algorithms in several factors: (1) The computational complexity of our method is smaller; (2) Our method is less sensitive to loose bounds on mixing times, an inherent component in these algorithms; and (3) The improvement obtained by our method is particularly significant in the most challenging regime of high precision estimates.
We demonstrate the advantage of our method in experiments run on classic factor graphs, such as voting and Ising models~\cite{cipra1987introduction,anari2020spectral,bhatnagar2010reconstruction}. 
\subsection{Preliminaries and Prior Work}\label{sc:prelim}

Assume a \emph{sample space} $\Omega$, \emph{Hamiltonian function} $H:\Omega\rightarrow \{0\}\cup [1,\infty)$, and \emph{inverse temperature} parameter $\beta \in \R$, referred to as \emph{inverse temperature}. The \emph{Gibbs distribution} on $\Omega$, $H(\cdot)$, and $\beta$ is then characterized by probability law
\begin{equation}\label{eq:gibbsdist}
    \forall x\in \Omega: \ \gibbs{\beta}(x) \doteq {\frac{1}{\Zi{\beta}}}\exp\bigl(-\beta H(x)\bigr) \enspace .
\end{equation}
Here $\Zi{\beta}$ is the \emph{normalizing constant} or \emph{Gibbs partition function} (GPF) of the distribution, with 
\begin{equation}\label{eq:normalizing}
\Zi{\beta} \doteq {\sum_{x\in \Omega}} \exp\bigl(-\beta H(x)\bigr) \enspace .\end{equation}


Estimating the GPF $\Zi{\beta}$, is computationally challenging, since typically the size of $\Omega$ is exponential in the number of
variables, and the values of random terms in the sum have large variance (due to the exponential). The following problem has been extensively studied, and is the focus of this paper.


\begin{problem}\label{prob:fpras} Given a domain $\Omega$, a Hamiltonian function $H$, and a parameter $\beta$, design a Fully Polynomial Randomized Approximation Scheme (FPRAS) for estimating  the partition function $\Zi{\beta} \doteq \sum_{x\in \Omega} \exp\bigl(-\beta H(x)\bigr)$.
In other words, for user-supplied $\varepsilon$, the task is to produce an estimate $\hat{Z}(\beta)$, such that with probability at least $1-\delta$, we have  $(1 - \varepsilon){Z}(\beta) \leq \hat{Z}(\beta) \leq (1 + \varepsilon){Z}(\beta)$, in time polynomial in $\nicefrac{1}{\varepsilon}$, $\ln(\nicefrac{1}{\delta})$, and all other problem parameters (e.g., the number of vertices in an Ising model, or neurons in an RBM).
\end{problem}


All known scalable solutions to this problem rely on Monte-Carlo Markov-chain (MCMC) methods, and their execution cost is dominated by the total number of Markov chain steps they execute.
We therefore follow past work, and analyze our 
algorithms in terms of number of the Markov chain steps. 

 \paragraph{TPA-Based Adaptive Cooling Schedules}
 Building on extensive earlier work~\cite{Fishman1994ChoosingSP,Bezkov2006AcceleratingSA,Stefankovich2009simulateannealing,Bezkov2006AcceleratingSA}, the current state of the art is due to
Huber and Schott~\cite{TPAhuber_schott_2014}, with  Kolmogorov's sharper analysis~\cite{Kolmogorovgibbs}.
They introduce the \emph{paired product estimator} (PPE), see \cref{defin:pairproduct}, and apply the \emph{tootsie-pop algorithm} (TPA) to adaptively compute a near-optimal \emph{cooling schedule}, i.e., a sequence of inverse temperatures  $\beta_0< \beta_1 < \dots < \beta_{\ell-1} < \beta_{\ell}$ satisfying $\beta_{\ell}=\beta$, and that $\Zi{\beta_0}$ is easy to compute, e.g., $\beta_{0} = 0$ is often convenient, since $\Zi{0}= \abs{\Omega}$. We thus define $Q\doteq\nicefrac{\Zi{\beta}}{\Zi{\beta_0}}$ and estimate it using the paired product estimator. 
\begin{defin}[PPE \cite{hubergibbs}]\label{defin:pairproduct}
Assume a \emph{cooling schedule} $\beta_{0}, \beta_{1}, \dots, \beta_{\ell}$. 
For each pair $(\beta_i,\beta_{i+1})$ in the schedule, we define two random variables, $X_i\sim \gibbs{\beta_i}$ and $Y_i\sim \gibbs{\beta_{i+1}}$, all independent, and we then define $f_{\beta_i,\beta_{i+1}}\doteq\exp\bigl(-\frac{\beta_{i+1}-\beta_{i}}{2}H(X_i)\bigr)$ and $g_{\beta_i,\beta_{i+1}}\doteq\exp\bigl(\frac{\beta_{i+1}-\beta_{i}}{2}H(Y_i)\bigr)$. 
It is easy to verify that 
$\mathbb{E}[f_{\beta_i,\beta_{i+1}}]= \Zi{\frac{\beta_i+\beta_{i+1}}{2}}/\Zi{\beta_i}$, and
$\mathbb{E}[g_{\beta_i,\beta_{i+1}}]= \Zi{\frac{\beta_i+\beta_{i+1}}{2}}/\Zi{\beta_{i+1}}$.
We then define $F\doteq\prod_{i=1}^k f_{\beta_i,\beta_{i+1}}$, 
$G\doteq\prod_{i=1}^k g_{\beta_i,\beta_{i+1}}
$. Letting $\hat{\mu}$ and $\hat{\nu}$ denote 
empirical estimates  of $\mathbb{E}[F]$ and $\mathbb{E}[G]$, respectively, the \emph{paired product estimator} (PPE) is $\hat{Q}\doteq \nicefrac{\hat{\mu}}{\hat{\nu}}$~.\cyrus{$\bigotimes$ ?}\shahrzad{I would stick to $\prod$ because $\otimes$ is introduced later to go with the chain}
%
\end{defin}




Denote by ${\mathbb V}_{\rm rel}[X]\doteq\nicefrac{\mathbb{E}[X^2]}{\mathbb{E}[X]^2}-1 =\nicefrac{{\mathbb V}[X]}{\mathbb{E}[X]^2}$ the \emph{relative variance} of a random variable $X$. 
\emph{The TPA schedule}~\cite{HuberTPA2010,TPAhuber_schott_2014} is generated by an adaptive
algorithm,  which, by a proper setting of parameters, outputs a cooling schedule guaranteeing  constant $\mathbb{V}_{\rm rel}[F]$ and $\mathbb{V}_{\rm rel}[G]$ (see alg.~3 in the supplementary material). 
 Kolmogorov \cite{Kolmogorovgibbs}  presents a tighter analysis of Huber's \tpa\ method, and proves that with  slight modifications (see alg. 4. in the Appendix) the schedule has a shorter length, 
 while preserving  constant relative variance for the paired product estimators (see \cref{remark2}).
In this paper, we use Kolmogorov's algorithm, 
and we denote it by \tpa$(k,d)$. For completeness, both of Huber's and Kolmogorov's versions of \tpa\ are presented in the Appendix. 

\cyrus{Check alg numbering}

We will use the following result in our analysis:

\begin{theorem}[\cite{Kolmogorovgibbs}]\label{remark2}
Let $H_{\max}\doteq \max_{x\in\Omega}H(x)$,
using \tpa$(k ,d)$, $k=\Theta(\log H_{\max})$ and $d=16$ to generate cooling schedule $(\beta_0,\beta_1,\dots ,\beta_{\ell})$.
W.h.p., we have
$\ell=\Theta(\log (Q) \log( H_{\max}))$ and
$\mathbb{V}_{\rm rel}[F]+1=\prod_{i=1}^{\ell} ({\mathbb{V}}_{\rm rel}[f_{\beta_i,\beta_{i+1}}]+1)=\Theta(1)$ and 
$\mathbb{V}_{\rm rel}[G]+1=\prod_{i=1}^{\ell} ({\mathbb{V}}_{\rm rel}[g_{\beta_i,\beta_{i+1}}]+1)=\Theta(1)$.
\end{theorem}


Kolmogorov \cite{Kolmogorovgibbs} nearly matches known lower bounds when given \emph{oracle access} to near-independent samples, but leaves open the possibility of  better use of the dependent sequence of samples generated by MCMC chains.
This fertile ground is ill-explored, since if an approximate sampling oracle draws samples by running a chain for $T$ steps, there is a factor $T$ potential improvement. 

\paragraph{MCMC Mean-Estimator} 
Huber and Schott~\cite{TPAhuber_schott_2014} assume  unit-cost  for exact sampling from each $\gibbs{\beta_i}$, and Kolmogorov~\cite{Kolmogorovgibbs} extends their analysis to include the complexity of generating \emph{approximate samples} with standard MCMC processes, assuming \emph{a priori} upper-bounds on their mixing times.
The main contribution of our paper is a specialized, adaptive, 
\emph{multiplicative} MCMC-mean estimator for the TPA-based  PPE. Our method is significantly more efficient than using standard black-box MCMC sampling for this problem, thus we improve the best-known method for estimating the GPF. 

Let $\cal M$ be an ergodic Markov chain with state space $S$ and stationary distribution $\pi$. Let $\TMix(\varepsilon)$ denote the $\varepsilon$-\emph{mixing time} of $\cal M$, and define $\TMix \doteq \TMix(\nicefrac{1}{4})$.
Letting $\lambda$ denote the
\emph{second largest absolute eigenvalue}  of $\M$'s transition matrix, the \emph{relaxation time} of $\cal M$ is $\TRel \doteq (1-\lambda)^{-1}$,  and it is  related to the mixing time $\TMix$, by $
\left(\TRel(\M)-1\right)\ln(2)\leq \TMix(\M)\leq \bigl\lceil \TRel(\M) \ln\bigl(\smash{\nicefrac{2}{\sqrt{\pi_{\min}}}}\bigr)\bigr\rceil$~\cite{levin2017markov}. Let $T$ be an upper bound on  $\max\{\TRel(\M) , \TMix(\M) \}$. 
Consider any i.i.d.\ sampling concentration bound like Chebyshev's, Hoeffding's, or Bernstein's inequalities \cite{mitzenmacher2017probability}, with, say, sample complexity $m_\varepsilon$. Using   MCMC as a black-box sampling tool, we obtain the same precision estimation guarantees, with a computational cost of  $m_\varepsilon\cdot \TMix(\varepsilon/m_\varepsilon)$, which is  equal to  $m_{\varepsilon}\log(m_{\varepsilon}\cdot \varepsilon^{-1})\cdot T$ in the absence  of exact values for $\TMix$.

Other concentration bounds compute the average  over the entire trace of a Markov chain,  and their complexity is dependent on \emph{known upper-bounds} on the \emph{relaxation time} 
\cite{paulin2015,mitzenmacher2017probability,LezaudChernoff,chung2012chernoff,jiang2018bernstein}, or function specific mixing time \cite{functionMixingRabinovish}. Note that since $\log(\frac{1}{2\varepsilon})(\TRel-1)\leq \TMix(\varepsilon)\leq \log(\frac{1}{\varepsilon \pi_{\min}})\TRel$, using these bounds is often more efficient, saving at least $\log(m_\varepsilon)$ steps.



Recently, Cousins et al.\ \cite{Dynamite} introduce a novel Markov chain statistical measure, the \emph{inter-trace variance}. The inter-trace variance depends on both the \emph{function being estimated} and the \emph{dependency structure} between nearby samples in the chain, and unlike the mixing time, it can be  efficiently estimated from data. 
By using progressive sampling, Cousins et al.\ show an \emph{additive MCMC mean estimator} whose complexity is proved in terms of \emph{inter-trace variance} and  they show it it less sensitive to  \emph{prior} knowledge of the input parameters, such as relaxation time and trace variance. Unfortunately due to a few technical problems, their result can not directly be used with the \tpa\ method. Thus, in order to obtain a doubly adaptive algorithm for \cref{prob:fpras}, we tailor their techniques to our setting, which requires developing new algorithms and analysis tools. 

\subsection{Our Main Contributions}

\vspace{-0.1cm}

\begin{itemize}[wide, labelwidth=0pt, labelindent=8pt]\setlength{\itemsep}{0pt}\setlength{\parskip}{0pt}
    \item We present a specialized mean estimator method that significantly improves the state of the art computational complexity of computing the partition function of Gibbs distribution. 
    \item While all rigorous MCMC-based estimates depend on some \emph{a priori} knowledge of the Markov chain properties (such as bounds on its mixing or relaxation time), the complexity of our method is less dependent on 
    these \emph{a priori} bounds, and decays gracefully as they become looser.
    \item The improvement of our method is particularly significant in the more challenging \emph{high precision} regime, where the goal is to compute estimates with very small multiplicative error.
    \item Our method improves the computational cost of prior work by replacing standard black-box MCMC mean estimators 
    with an 
    \emph{adaptive MCMC estimator}, specially tailored to this problem.
    \item The analysis of our method relies on a novel notion of sample variance in a sequence of observations obtained by Markov chains runs, which we term the \emph{relative trace variance}.
    \item We demonstrate the practicality of our 
    method through experiments on 
    Ising and voting models. 
\end{itemize}
\section{Algorithms}\label{sec:algo}
\vspace{-0.2cm}
In this section, we develop two \emph{doubly-adaptive} \emph{fully polynomial randomized approximation schemes} providing more efficient algorithmic solutions to \cref{prob:fpras}. The proof of all of the lemmas and theorems are presented fully in the supplementary material. 

\paragraph{Notation and Setting Parameters} We use the following notation throughout: We use capital letters to denote upper-bounds. e.g., $T$  denotes an upper-bound on $\max(\TMix,\TRel)$, and $\Lambda$  denotes a upper-bound on the second absolute eigenvalue $\lambda$. We use $\gibbschain{H}{\beta}$ to denote any Markov chain with Gibbs stationary distribution $\gibbs{\beta}$, eq. \eqref{eq:gibbsdist}. 
Having the Hamiltonian $H$, we denote its maximum and minimum values as $H_{\rm max}$ and $H_{\rm min}$, i.e.,
$H_{\rm max}\doteq \max_{x\in\Omega}\{H(x)\}$ and $H_{\rm min}\doteq \min_{x\in\Omega}\{H(x)\}$. 
Having a schedule $(\beta_0,\beta_1,\dots,\beta_{\ell})$, the paired product estimators $f_{\beta_i,\beta_{i+1}}$, $g_{\beta_i,\beta_{i+1}}$, $F=\bigotimes_{i=1}^{\ell} f_{\beta_i,\beta_{i+1}}$ and
$G=\bigotimes_{i=1}^{\ell} g_{\beta_i,\beta_{i+1}}$
are as in \cref{defin:pairproduct}.
When writing $(\beta_0,\beta_1,\dots,\beta_{\ell})=\tpa(k,d)$, we mean the cooling schedule is obtained from running alg.\ 4 in the Appendix, and we always set $k=\log H_{\rm max}$ and $d=64$, as these parameters are shown to produce a near-optimal schedule w.h.p.\ \cite{Kolmogorovgibbs}.

We first introduce a novel MCMC-based \emph{multiplicative} mean estimation procedure \proc\ (see \cref{alg:meanestsubroutine}), and analyze its computational complexity in terms of a new quantity, which we coin \emph{the relative trace variance} (see \cref{def:reltracevariance}). \proc\ receives as input a Markov chain $\M$, a function $f$, and precision parameters $\varepsilon$ and $\delta$, and it outputs a multiplicative estimate of the expected value of the function w.r.t.\ the stationary distribution of the Markov chain. For simplicity, we may refer to it as \proc$(\M,f)$, leaving out the precision parameters.  

Letting $(\beta_0,\beta_1,\dots ,\beta_{\ell})=\tpa(k,d)$,
we first present \paralgo, in which we invoke both \proc$(\gibbschain{H}{\beta_i},f_{\beta_i,\beta_{i+1}})$ and \proc$(\gibbschain{H}{\beta_i},g_{\beta_i,\beta_{i+1}})$ for each $i=1,2,\dots , \ell-1$.
We then present an often-more-efficient 
algorithm, \superalgo, which invokes \proc\ once each on $F$ and $G$ on a \emph{``super'' product chain}  (see \cref{def:prod}). We prove correctness of  both \paralgo\ and \superalgo, and bound their complexity in terms of \emph{the relative trace variance} of the estimators.
Furthermore, we prove \superalgo{}\cyrus{Just \superalgo?}\shahrzad{we only PROVE superchain is better. Both are better in experiments} improves the computational complexity of the state of the art \cite{Kolmogorovgibbs} (\cref{thm:superchainefficiency} and \cref{coro:smallep}). Both of these algorithms have low dependence on tightness of mixing time: They receive as input an upper-bound on mixing or relaxation time $T$, but we show for $\varepsilon\geq \varepsilon_0$ their computation complexity is dominated by the \emph{true relaxation time} $\tau_{\rm rel}$ (of each Gibbs chain or the product chain).
\subsection{Relative trace variance and \proc}\label{sec:mcmc}

In this section we introduce a new variance notion, the \emph{relative trace variance}, which captures the computational complexity of MCMC-mean estimation with \emph{multiplicative} precision guarantees.  The \emph{relative trace variance} depends on both the chain $\M$ and the function $f$, 
and it generalizes the \emph{relative variance}, defined as ${\mathbb V}_{\rm rel}[f]\doteq\nicefrac{{\mathbb V}[f]}{\mathbb{E}[f]^2}$, which depends only on $f$, and is used in i.i.d.\ regimes.



\begin{defin}[Relative Trace Variance]\label{def:reltracevariance}

For arbitrary $\tau$, consider a trace of length $\tau$ of a Markov chain $\cal M$, and a real-valued function $f$. On $\cal M$, we define the relative trace variance of $f$ as 
$$\RelITRelVar{{\cal M}}{\tau}[f] \doteq \frac{\mathbb{E}[{\bar{f}(\vec{X}_{1:\tau})}^2]}{(\mathbb{E}[\bar{f}(\vec{X}_{1:\tau})])^2}-1 \enspace,
$$
where $\vec{X}_{1:\tau} \doteq X_1,X_2,\dots ,X_{\tau}$ is a trace of length $\tau$ of $\cal M$, and $\bar{f}(\vec{X}_{1:\tau}) \doteq (\frac{1}{\tau})\sum_{i=1}^{\tau} f(X_i)$. We may drop the subscript when the chain is clear from the context. 
\end{defin}

The above definition is similar to what Cousins et al.\ coined as \emph{the inter-trace variance}, denoted by ${\rm trv}^{(\tau)}(\M,f)$, which they showed it captures  MCMC-mean estimation with \emph{additive} precision guarantees \cite{Dynamite}. In fact, the two terms are related as 
\[\RelITRelVar{\cal M}{\tau}[f]=\frac{{\rm trv}^{(\tau)}(\M,f)}{(\mathbb{E}[\bar{f}(\vec{X}_{1:\tau})])^2}\enspace .\]
Note that the two terms are not easily convertible without knowing the \emph{mean}, $\mathbb{E}[\bar{f}(\vec{X}_{1:\tau})]$.

\begin{lemma}\label{remark}
For any $\tau$ we have 
\begin{equation}\label{eq:relvar1}
\RelITRelVar{\cal M}{\tau}[f]\leq  \mathbb{V}_{\rm rel}[f] \enspace.
\end{equation}
Furthermore, for $\tau\geq \TRel({\cal M})$ we have, 
\begin{equation}\label{eq:relvar2}
\RelITRelVar{\cal M}{\tau}[f]= O\left(\frac{\TRel({\cal M})}{\tau}\RelITRelVar{\cal M}{\TRel({\cal M})}[f]\right) \enspace. \end{equation}

\end{lemma}

 \Cref{remark} enables us to compare the computational complexity of our algorithms with the state of the art \cite{Kolmogorovgibbs}. In particular, using \eqref{eq:relvar1}, we  show our results improve the state of the art  (which is in terms of $\mathbb{V}_{\rm rel}$), and using \eqref{eq:relvar2}, we show that for high-precision estimations, the sample complexity of our algorithms  only depends on $\TRel$, which improves the state of the art (which is in term of $T$).


The relative trace variance is a better analysis tool for estimating the GPF, because, unlike the inter-trace variance, it leads directly to \emph{relative error bounds}, rather than \emph{absolute error bounds}.
 We now present some definitions  which can also be found in standard MCMC textbooks, e.g., \cite{levin2017markov}.

\begin{defin}[Product Chain and Tensor Product Function]\label{def:prod} Consider  $k$  
 Markov chains $\{{\cal M}_i\}_{i=1}^k$  each defined on state space $S_i$ and assume real valued functions $\{f_i:S_i\rightarrow {\mathbb R}\}_{i=1}^k$.  The \emph{product chain} ${\cal M}^{\otimes}_{1:k}$ is defined on the Cartesian product of $S_i$ as follows: 
 at any step ${\cal M}^{\otimes}_{1:k}$ chooses $ i$ with probability $\omega_i$ (thus $\sum_{i=1}^k \omega_i=1) $, and moves from $(x_1,x_2,\dots, x_i,\dots , x_k)$ to $(x_1,x_2,\dots, y_i,\dots , x_k)$, with the transition probability of moving from $x_i$ to $y_i$ 
 in  ${\cal M}_i$. The \emph{tensor product} of $\{f_i\}_{i=1}^k$, denoted by $\bigotimes_{1:k}f_i$, 
 is defined as $\bigl(\bigotimes_{1:k}f_i\bigr)(x_1,x_2,\dots, x_k)=\prod_{i=1}^k f_i(x_i)$.
\end{defin}

\paragraph{\proc}
Let $T$ denote an upper bound on the relaxation time of a Markov chain $\M$. \proc\ receives  $T$, $\M$, $f$ and precision parameters $\varepsilon$ and $\delta$ as input. Before it starts collecting samples, it runs the chain for a \emph{warm start} (\cref{alg:rme:warm} of \cref{alg:meanestsubroutine}).
 Starting from a minimum sample size $m^{\downarrow}$, it runs $\M$ for $T\cdot m^{\downarrow} $ steps, and collect samples $X_1,X_2,\dots, X_{T\cdot m^{\downarrow}}$. It then computes for $j=1,2,\dots, m^{\downarrow}$, $\bar{f}_j\doteq   \sum_{i=(j-1)\cdot T+1}^{j\cdot T}f(X_i)$; using them, it calculates an empirical estimate of the mean, $\hat{\mu}$, and an empirical estimation for the \emph{trace variance}  of $\M$ and $f$, $\hat{v}$. 
 Based on these estimates, we derive an upper-bound on the current trace variance $u_{i}$ and relative error $\hat{\bm{\varepsilon}}_i^{\times}$, and check whether is smaller than the user-specified error $\varepsilon$ (lines~18-19).
If so, we return the current mean estimate, otherwise we double the sample size and repeat.

\begin{algorithm}
\algrenewcommand\algorithmicindent{1.0em}
\scalebox{0.75}{\begin{minipage}{1.33333\textwidth}
\begin{algorithmic}[1]
\newcommand{\VEEpsAdd}{\hat{\bm{\varepsilon}}^{+}}
\newcommand{\VEEpsMult}{\hat{\bm{\varepsilon}}^{\times}}
\newcommand{\VEMean}{\hat{\bm{\mu}}}
\Procedure{\proc}{}
\State \textbf{Input:} Markov chain $\M$, upper-bound on relaxation time $T$, real-valued function $f$ with range $[a, b]$, letting $R = b - a$, multiplicative precision $\varepsilon$, error probability $\delta$.

\State \textbf{Output:} Multiplicative approximation $\hat{\mu}$ of $\mu=\Expect_{\pi}[f]$.

\medskip 

\State  $T \gets \ceil{\frac{1 + \EigTwoBound}{1 - \EigTwoBound} \ln\sqrt{2} }$; $\EigTwoBound' \gets \EigTwoBound^{T}$ \Comment{Choose $T$ to be an upperbound on relaxation time}\label{alg:rme:trel} 


\State $\NIterations \gets 1 \vee  \left \lfloor \log_{2}\left( \frac{b\frange}{2 a^{2}} \cdot \frac{ (1-\varepsilon)^{2}}{ (1 + \varepsilon)\varepsilon} \right) \right \rfloor; \ \alpha \gets \frac{(1+\EigTwoBound')\frange \ln \frac{3I}{\delta} (1 + \varepsilon)}{(1-\EigTwoBound')b \varepsilon}
; \ m_{0} \gets 0$ \Comment{Initialize \emph{sampling schedule} 
}\label{alg:niter}\label{alg:alpha}

\State $T_{\rm unif} \gets \ceil{T \cdot \ln(\nicefrac{1}{\pi_{\rm min}})}$; $(\vec{X}_{0,1}, \vec{X}_{0,2}) \gets \M^{T_{\rm unif}}(\bot)$ \Comment{Warm-start two chains 
for $T_{\rm unif}$ steps from arbitrary $\bot \in \Omega$}\label{alg:rme:warm}
\cyrus{Need $\delta$ correction.}

\For{$i \in 1, 2, \dots, \NIterations $}
\State $m_i \gets \left\lceil \alpha r^i \right\rceil$ \Comment{Total sample count at iteration $i$; $r$ is the geometric ratio (constant, usually 2) size\cyrus{explain r, how this impacts schedule}}\label{alg:ss}
\For{$j \in (m_{i-1} + 1), \dots, m_{i}$}


\State $(\vec{X}_{j,1}, \vec{X}_{j,2}) \gets (T$ steps of $\M$ starting at   $\vec{X}_{j-1,1}, \vec{X}_{j-1,2})$
\Comment{Run two independent copies of  $\M$ for $T$ steps}

\State $\displaystyle \bar{f}(\vec{X}_{j,1})\gets \frac{1}{T}\sum_{t=1}^T f\bigl(\vec{X}_{j,1}(t)\bigr)$; $\displaystyle\bar{f}(\vec{X}_{j,2})\gets \frac{1}{T}\sum_{t=1}^T f\bigl(\vec{X}_{j,2}(t)\bigr)$
\Comment{Average $f$ over $T$-traces} 
\smallskip 

\EndFor



\State $\displaystyle {\VEMean}_{i} \gets \frac{1}{2m_i} \sum_{i=1}^{m_{i}}\bigl( f(\vec{X}_{j,1})+f(\vec{X}_{j,2})
\bigr)$; 
$\displaystyle\hat{v}_{i} \gets \frac{{1}}{2{m_i}}  \sum_{i=1}^{m_{i}}\bigl(( f(\vec{X}_{j,1})-f(\vec{X}_{j,2})
\bigr)^2$ \Comment{Compute empirical mean; trace variance}\label{alg:emean} \label{alg:evar} \label{alg:emean-evar}

\State $\displaystyle u_{i} \gets \hat{v}_{i} + \frac{(11 + \sqrt{21})(1 + 
  \nicefrac{\EigTwoBound'}{\sqrt{21}}
) \frange^{2} \ln \frac{3\NIterations}{\delta}}{(1-\EigTwoBound')m_{i}} + \sqrt{\frac{(1 + \EigTwoBound')\frange^{2} \hat{v}_{i}\ln \frac{3\NIterations}{\delta}}{(1 - \EigTwoBound')m_{i}}}
$\Comment{Variance upper bound}\label{alg:mcd-var}\label{alg:2chain-var}
\State $\displaystyle \VEEpsAdd_{i} \gets \frac{10\frange\ln \frac{3\NIterations}{\delta}}{(1-\EigTwoBound')m_{i}} + \sqrt{\frac{(1 + \EigTwoBound') u_{i}\ln \frac{3\NIterations}{\delta}}{(1 - \EigTwoBound')m_{i}}}$ \Comment{Apply Bernstein bound}\label{alg:bern}





\State $\displaystyle \VEMean^{\times}_{i} \gets \frac{(\VEMean_{i} - \VEEpsAdd_{i}) \vee a + (\VEMean_{i} + \VEEpsAdd_{i}) \wedge b}{2}$
\Comment{Optimal mean estimate}


\State  $\displaystyle \VEEpsMult_{i} \gets \frac{((\VEMean_{i} + \VEEpsAdd_{i}) \wedge b- (\VEMean_{i} - \VEEpsAdd_{i}) \vee a }{2\VEMean^{\times}_{i}}$ \Comment{Empirical relative error bound}

\If{$(i = \NIterations) \vee (\VEEpsMult_{i} \leq \varepsilon)$} \Comment{Terminate if accuracy guarantee is met}\label{alg:tc}

\State \Return $\VEMean_{i}^{\times}$ 
 \label{alg:return}
\EndIf
\EndFor
\EndProcedure
\end{algorithmic}
\end{minipage}}
\caption{\proc}
\label{alg:meanestsubroutine}
\label{algo:mean}
\end{algorithm}

The following theorem,
shows the correctness of \proc\ and bounds its  complexity. 

\begin{restatable}[Efficiency and Correctness of \proc]{theorem}{thmefficiency}
\label{thm:efficiency}
With probability at least $1 - \delta$, 
\proc\ will output $\hat{\mu}$ satisfying $(1 - \varepsilon)\hat{\mu} \leq \mu \leq (1 + \varepsilon)\hat{\mu}$.
Furthermore, with probability at least $1 - \frac{\delta}{3\NIterations}$, the total Markov chain steps of \proc, $\hat{m}$, obeys
\begin{equation}\label{eq:proc}
\hat{m} \in \mathcal{O}\left( \ln\left(\frac{\ln\frac{b}{a\varepsilon}}{\delta}\right)\left( \frac{T\cdot \frange}{\mu\varepsilon} + \frac{\TRel\RelITRelVar{}{\TRel}}{\varepsilon^{2}}\right)\right).
\end{equation}
\end{restatable}

\subsection{Doubly adaptive algorithms: \superalgo\ and \paralgo} \label{sec:superalgo}
Let 
 $(\beta_0,\beta_1,\dots, \beta_{\ell})=\tpa(k, d)$, 
 and consider a family of Gibbs chains $\gibbschain{H}{\beta_i}$, each corresponding to some $\beta_i$, and the  paired product estimators $F=\bigotimes_{i=1}^{\ell} f_{\beta_i,\beta_{i+1}}$  $G=\bigotimes_{i=1}^{\ell} g_{\beta_i,\beta_{i+1}}$. The \tpa\ method is designed to ensure $\mathbb{V}_{\rm rel}$ of the estimators are bounded, which can be employed by  
 concentration bounds (e.g., Chebyshev's bound) to guarantee the multiplicative error is bounded with high probability for a given sample size.
 
 In order to generalize the same machinery for  samples generated from a Markov chain  using \proc,\ we need to bound the two terms appearing in \cref{eq:proc}, which  dominate the computational complexity of \proc\cyrus{Only for analysis, not to use the machinery?}. 
 We refer to the first term, $T\cdot \frange/\mu$, as the \emph{range term}, and to the term $\TRel \RelITRelVar{}{\TRel}$  as the \emph{trace variance term}. Note that as $\varepsilon$ becomes smaller, the \emph{trace variance} term dominates the sample complexity of \proc, thus dependence on  loose bounds $T$ and $R$ is dominated by 
 dependence on \emph{true and a priori unknown} values  $\TRel$ and $ \RelITRelVar{}{\TRel}$.

In order to ensure that the ranges of estimators are small, we prove that the length of each inverse-temperature interval in the \tpa\ schedule is w.h.p. small.  Having a schedule $(\beta_0,\beta_1,\dots ,\beta_{\ell})$ we define and use the following notation:
 for $0\leq i\leq \ell-1$, \emph{interval length} $\Delta_i\doteq \beta_{i+1}-\beta_i$, 
 \emph{maximum interval length} $\Delta_{\max}\doteq \max_{i}\Delta_{i}$, and \emph{total length}
 $\Delta \doteq \beta_{\ell}-\beta_0$~. 
 
   \begin{restatable}{lemma}{lemintervalbd}
    \label{lem:intervalbd}
   Let $z(\beta)\doteq\ln\left(\Zi{\beta}\right)$, and let $\beta_i$, $\beta_{i+1}$ be two consecutive points generated by $\tpa(k,d)$. 
   For arbitrary $\varepsilon\geq 0$, we have:
    \vspace{-0.2cm}
    \begin{enumerate}[wide, labelwidth=0pt, labelindent=0pt]\setlength{\itemsep}{0pt}\setlength{\parskip}{0pt}
        \item   
        $\mathbb{P}(z(\beta_i)-z(\beta_{i+1})\leq \varepsilon)\geq (1-\exp(-\varepsilon k/d))^d$.
        \item  
        $\mathbb{P}\left(\Delta_i\geq \nicefrac{\varepsilon}{\mathbb{E}[H(x)]}\right)\leq  d \exp(-\varepsilon k/d),$ where $\mathbb{E}[H(x)]$ is taken w.r.t.\ $x\sim \gibbs{\beta_{i+1}}$. 
    \end{enumerate}
    \end{restatable}

\paragraph{\superalgo}
 Let $\cal {\cal G}^{\otimes}$ the product of $\gibbschain{H}{\beta_i}$s with uniform weights i.e., $\omega_i=\frac{1}{\ell}, \forall i$ (see \cref{def:prod}). 
 \superalgo\ calls \proc$({\cal {\cal G}^{\otimes}},F)$ and \proc$({\cal {\cal G}^{\otimes}},G)$, with appropriate parameters, and simply outputs the ratio of the two estimates (see \cref{algo:supergocode}, left).

 \begin{algorithm}
\algrenewcommand\algorithmicindent{0.5em}

\scalebox{0.9}{\begin{minipage}{1.111111111\textwidth}
\setlength{\columnsep}{-0.6cm}
\begin{multicols}{2}
\centering

\begin{algorithmic}[1]
\small
\Procedure{\superalgo}{\dots}

\renewcommand*{\thefootnote}{\fnsymbol{footnote}}

\State $(\beta_0, \beta_1 ,\dots ,\beta_\ell) \gets \tpa(k,d)$\footnote{ $k=\Theta(\log H_{\max})$ and $d=64$ as in \cite{Kolmogorovgibbs}}
\State $\varepsilon'\gets \frac{\varepsilon}{2 + \varepsilon}$; $\delta'\gets \frac{\delta}{2}$
\For{$i \in 1,2,\dots, \ell $}
\State $f_i(x) \doteq \exp(-\frac{\beta_{i+1}-\beta_i}{2}H(x))$
\State $g_{i}(x) \doteq \exp(\frac{\beta_{i}-\beta_{i-1}}{2}H(x))$

\EndFor
\State $F \doteq \bigotimes_{i=1}^\ell f_i$; $G \doteq \bigotimes_{i=1}^\ell g_i$

\State ${\mathcal G}^{\otimes} \gets \bigotimes_{i=1}^\ell \gibbschain{H}{\beta_i}$, with $\omega_i= \frac{1}{\ell}, \forall i$
\State $R_f \gets \exp(-\frac{\beta-\beta_0}{2}H_{\min})-\exp(-\frac{\beta-\beta_0}{2}H_{\max})$
\State $R_g \gets \exp(\frac{\beta-\beta_0}{2}H_{\max})-\exp(\frac{\beta-\beta_0}{2}H_{\min})$

\State 
$\hat{\mu} \gets \proc({\mathcal G}^{\otimes},  R_f,T, F, \varepsilon', \delta')$
\State$\hat{\nu} \gets \proc({\mathcal G}^{\otimes}, R_g, T, G, \varepsilon', \delta')$
\State \textbf{return} $\hat{Z} \gets \frac{\hat{\nu}}{\hat{\mu}}$

\EndProcedure

\vfill\columnbreak

\Procedure{\paralgo}{\dots} \small


\State $(\beta_0, \beta_1 ,\dots ,\beta_\ell)= \tpa(k,d)$
\State $ \varepsilon' \gets \frac{\sqrt[\ell]{1 + \varepsilon} - 1}{\sqrt[\ell]{1 + \varepsilon} + 1}$; $\delta'\gets \frac{\delta}{2\ell}$
\For{$i \in 1,2,\dots \ell $}
\State $f_i(x) \doteq \exp(-\frac{\beta_{i+1}-\beta_i}{2}H(x))$
\State $g_{i-1}(x) \doteq \exp(\frac{\beta_{i}-\beta_{i-1}}{2}H(x))$
\State $R_f \gets \exp(-\frac{\beta_{i+1}-\beta_i}{2}H_{\min})-\exp(-\frac{\beta_{i+1}-\beta_i}{2}H_{\max})$
\State $R_g \gets \exp(\frac{\beta_{i+1}-\beta_i}{2}H_{\max})-\exp(\frac{\beta_{i+1}-\beta_i}{2}H_{\min})$
\State  $\hat{\mu}_i \gets \proc({\cal G}_i,   R_f, T_i, f_i, \varepsilon', \delta')$
\State $\hat{\nu}_{i} \gets \proc({\cal G}_i,  R_g,T_i, g_{i}, \varepsilon', \delta')$
\EndFor
 \State \textbf{return} $\hat{Z} \gets 
 {\prod_{i=1}^{\ell} \frac{\hat{\nu}_i}{\hat{\mu}_i}}$

\EndProcedure
\end{algorithmic}

\cyrus{Check for flipped order?}
\cyrus{alg arguments, check order, all algos}

\end{multicols}
\end{minipage}
}

\caption{{\superalgo~ and \paralgo} }
\label{algo:supergocode}
\end{algorithm}


 We now show the correctness and efficiency of \superalgo. 
 Let $\tau_{\rm prx}$ denote ${\mathcal G}^{\otimes}$'s true (and unknown) relaxation time and $T$ a known upper-bound on it ($T\geq \tau_{\rm prx}$), $\varepsilon$ and $\delta$ are user specified precision parameters. For simlicity of presentation we use the following notation to refer to relative ranges: ${\rm relR}={\rm Range}(F)/\mu+{\rm Range}(G)/\nu$, where $\mu=\Expect[F]$ and $\nu=\Expect[G]$~.

\begin{theorem}
\label{thm:superchainefficiency}
With probability at least $1 - \delta$, 
it holds that the total number $\hat{m}$ of Markov chain steps  taken by \superalgo\ is upper-bounded by
\[
 \mathcal{\tilde{O}}\Biggl( \ln\biggl(\frac{1}{\delta}\biggr)\Biggl( \frac{T\cdot{\rm relR} }{\varepsilon} + \frac{{\tau_{\rm prx}}\cdot \bigl(\RelITRelVar{{\cal G}^\otimes}{\tau_{\rm prx}}(F)+\RelITRelVar{{\cal G}^{\otimes}}{\tau_{\rm prx}}(G)\bigr)}{\varepsilon^{2}}\Biggr)\Biggr)
~.\]
\end{theorem}
 \begin{lemma}\label{lem:TPAboundsSuperchain}
Defining 
$\alpha_1=\sqrt{\frac{Z(\beta_{0})}{Z(\beta_{0}-\Delta_{\max})}} $,  
we have: 
    $\hspace{1.5 cm}\frac{{\rm Range}(F)}{\mu}\leq \alpha_{1}\sqrt{\frac{Q}{\exp(\Delta H_{\min})}}$ and $   \frac{{\rm Range}(G)}{\nu}\leq \alpha_{1}\sqrt{\frac{\exp(\Delta H_{\max}) }{Q}} $~.
\end{lemma}
Using \cref{lem:TPAboundsSuperchain} and \cref{thm:superchainefficiency}, we identify $\varepsilon_0$ such that for $\varepsilon\leq \varepsilon_0$ the trace variance term in will the dominate computational complexity of \superalgo. In order to make a fair comparison with the state of the art \cite{Kolmogorovgibbs} we employ \cref{eq:relvar1} of \cref{remark}. Finally we use \cref{remark2} and conclude: 

\begin{coro}\label{coro:smallep} Let $\alpha_1$ 
be as in \cref{lem:TPAboundsSuperchain}, $\tau_{\max}\doteq \max_i \tau_i$ and $\varepsilon_0 \doteq ({\tau_{\rm prx}}/T)\cdot  \left(\sqrt{\frac{\exp(\Delta H_{\min})}{Q}}+\sqrt{\frac{Q}{\exp(\Delta H_{\max})}}\right)\cdot\alpha_1$.
When $\varepsilon \leq \varepsilon_0$, the number of Markov chain steps of \superalgo\ is dominated by $\tilde{O}(\ell  \tau_{\max})$~.
\end{coro}

\paragraph{\paralgo}

For $i=1,2,\dots, \ell-1$, \paralgo\ (\cref{algo:supergocode}, right) 
runs $\proc(\gibbschain{H}{\beta_i},f_{\beta_i,\beta_{i+1}})$ and 
$\proc(\gibbschain{H}{\beta_i},g_{\beta_i,\beta_{i+1}})$ independently. We show the computational complexity of \paralgo\ in \cref{thm:parchainefficiency}.

For $i=1,2,\dots, \ell$, assume $\tau_i$ is the true (unknown) relaxation time of $\gibbschain{H}{\beta_i}$ and $T_i$ is a known bound on it. For simplicity of presentation we use the following notations: ${\rm relR}_i\doteq {\rm Range}(f_{\beta_i,\beta_{i+1}})/\mu_i+{\rm Range}(g_{\beta_{i-1},\beta_i})/\nu_i$, where $\mu_i=\mathbb{E}(f_{\beta_i,\beta_{i+1}})$ 
and $\nu_i=\mathbb{E}(g_{\beta_i,\beta_{i+1}})$.

\begin{restatable}[Efficiency of \paralgo{}]{theorem}{thm:parchain}
\label{thm:parchainefficiency}
With probability at least $1 - \delta$, 
it holds that the total number  $\hat{m}$ of Markov chain steps taken by  \paralgo{} is upper-bounded by
\begin{align*}
 \mathcal{\tilde{O}}\Biggl( \log\left(\frac{\ell
}{\delta}\right)\sum_{i=1}^{\ell}\left( \frac{\ell\cdot{T_i\cdot \rm relR}_i}{\varepsilon} +\frac{\ell^2 }{\varepsilon^2}  \tau_{i}\cdot \left(\RelITRelVar{\gibbschain{H}{\beta_i}}{\tau_i}(f_{\beta_i,\beta_{i+1}})+\RelITRelVar{\gibbschain{H}{\beta_i}}{\tau_i}(g_{\beta_{i-1},\beta_{i}})\right) \right)\Biggr)~.\\
\end{align*}
%
Furthermore, for all $1\leq i\leq \ell$, ${\rm Range}(f_{\beta_i,\beta_{i+1}})/\mu_i\leq \ell^{1/\log (n)}$ and  
 ${\rm Range}(g_{\beta_{i-1},\beta_i})/\nu_i\leq \ell^{\alpha_0(i)/\log n}$, where $\alpha_0(i)=(\nicefrac{ H_{\max}}{2\mathbb{E}[H(x)]})-1,$ for $x\sim{\gibbs{\beta_i}}~$.
\end{restatable}

\paralgo\ and \superalgo\ make different computational complexity tradeoffs. \paralgo\ is  usually slower  than \superalgo,\ because in each iteration $i=1,2,\dots,\ell$, the mean estimator must acquire a higher-precision estimate so that \emph{all estimators together} achieve an $\varepsilon$-$\delta$ relative-error guarantee.   Relaxation times (true values and their upper-bounds) appear \emph{in a sum} in the complexity of \paralgo, whereas they appear \emph{in a maximum} in \superalgo\ ($\sum_{i=1}^{\ell}\tau_i$ vs.\ $\max_{i=1, \dots, \ell} \tau_i$). Furthermore, dominance of the trace variance terms in both of these algorithms  occur at different values of $\varepsilon$. A comparison of the complexity of these algorithms, in the high-precision regime, with Kolmogorov's \TPA\ + PPE (which uses MCMC as a black box) is presented in \cref{tab:comparisioncon}.

\begin{table*}[t]

\scriptsize
\begin{tabular}{|c|c|c|}
\hline 
\paralgo&\superalgo& \tpa\ + {\rm PPE} \cite{Kolmogorovgibbs}\\
\hline 
\!\!\!$\displaystyle \ell^2\sum_{i=1}^\ell \tau_i 
\left({\rm Reltrv}^{\tau_i}_{\gibbschain{H}{\beta_i}}(f_{i}) + \null\right.$ \!\!\!
& \!\!\!$\displaystyle \tau_{\rm prx} \left( {\rm Reltrv}^{\tau_{\rm prx}} [F]+{\rm Reltrv}^{\tau_{\rm prx}} [G] \right)$\!\!\!&\!\!\!$\displaystyle\ln \frac{q \ln H_{\max}}{\varepsilon} \sum_{i=1}^\ell{T_i}\cdot  \left({\mathbb{V}}_{\rm rel}(F)+{\mathbb{V}}_{\rm rel}(G)\right)$\!\!\!\\
$\left.{\rm Reltrv}^{\tau_i}_{\gibbschain{H}{\beta_i}}(g_{i})  \right)$& $\displaystyle =O\left(\ell \max\{\tau_{i}\}_{i=1:\ell}  \right)$&$\displaystyle =O \left( \ln\frac{q\ln H_{\max}}{\varepsilon} \sum_{i=1}^\ell{T_i} \right)$\\
\hline 
\end{tabular}


\caption{\footnotesize{Comparison of the number of Markov chain steps, when  $\varepsilon$ is adequately small. In all columns, a multiplicative factor of $\nicefrac{1}{\varepsilon^2}$ is omitted to ease presentation, and $q=\ln Q$.
Note that computational complexity of both \paralgo\ and \superalgo\ only depends on true relaxation times, denoted by $\tau_i$, and the 
\TPA\ + PPE method's complexity is dependent on their upper bounds, denoted by $T_i$. }
}\label{tab:comparisioncon}
\end{table*}

\label{sec:superalgo}

\iftrue

\section{Experimental Results}\label{sec:exp}
In this section we report our experiment results, comparing 
the performance of the two versions of our \emph{doubly adaptive} method (alg. 2), to the performance  of the state of the art algorithm in~\cite{Kolmogorovgibbs}. 

{\bf Setup.} We run the experiments using the single site Gibbs sampler (known also as the Glauber dynamics) on two different factor graph models:

{\bf (A) The Ising model on 2D lattices.} 
Having a 2-dimension lattice of size $n\times n$, the Hamiltonian is defined on $n^2$ random variables having values $\pm 1$ and their dependency is represented by the Hamiltonian:
$
H(x) = -\sum_{(i,j)\in E}{\mathbbm{1}(x(i)=x(j))}
$. 
We run the algorithms on lattices of sizes $2\times 2$, $3\times 3$, $4\times 4$, and $6\times 6$. For each lattice, the parameter $\beta \geq 0$ is chosen below the critical inverse temperature at which it  undergoes a phase transition. We use known mixing time bounds  for high temperature Ising models \cite{aldous2013probability} (see \cref{fig:ising2} and A.6. of supplementary material).

 {\bf
 (B) The logical voting model.} For a parameter $n$ \cyrus{$n$ parameter name?}, we have $2n+1$ random variables:  the query variable $Q\in\{-1,1\}$, and the voter variables $T_1,T_2,\dots, T_n$ and $F_1,F_2,\dots, F_n$ all in $\{0,1\}$. The factors have $2n+1$ weights, $\omega, \omega_{T_i}, \omega_{F_i}, i=1,\dots, n$. 
The Hamiltonian is:
\vspace{-0.25 cm}
\begin{align*}
H(Q, T, F) = \omega Q \max_{i}{T_i} - \omega Q \max_{i}{F_i} +
\sum_{i=1}^{n}\omega_{T_i}T_i + \sum_{i=1}^{n}\omega_{F_i}F_i 
\textrm{ , where } \omega, \omega_{T_i}, \omega_{F_i} \in [-1, 1]
\end{align*}
The parameters are reported in \cref{fig:voting}.
We follow De Sa et al.\ \cite{DeSa2015GIBBSHW} and use \emph{hierarchy width} to derive upper bounds on mixing times.
To make a fair comparison, we always run the \tpa\ algorithms once, and with the parameters given in \cite{Kolmogorovgibbs}. At each iteration of \proc, the sample size is extended with geometric ratio $1.1$ (see~ \cref{alg:meanestsubroutine}~line~\ref{alg:ss}).
All code is available at \url{https://github.com/zysophia/Doubly_Adaptive_MCMC}.

\begin{figure*}[h]
  \scalebox{0.875}{
  \begin{subfigure}{0.3\textwidth}
      \centering
    \includegraphics[width=6cm,height=6cm]{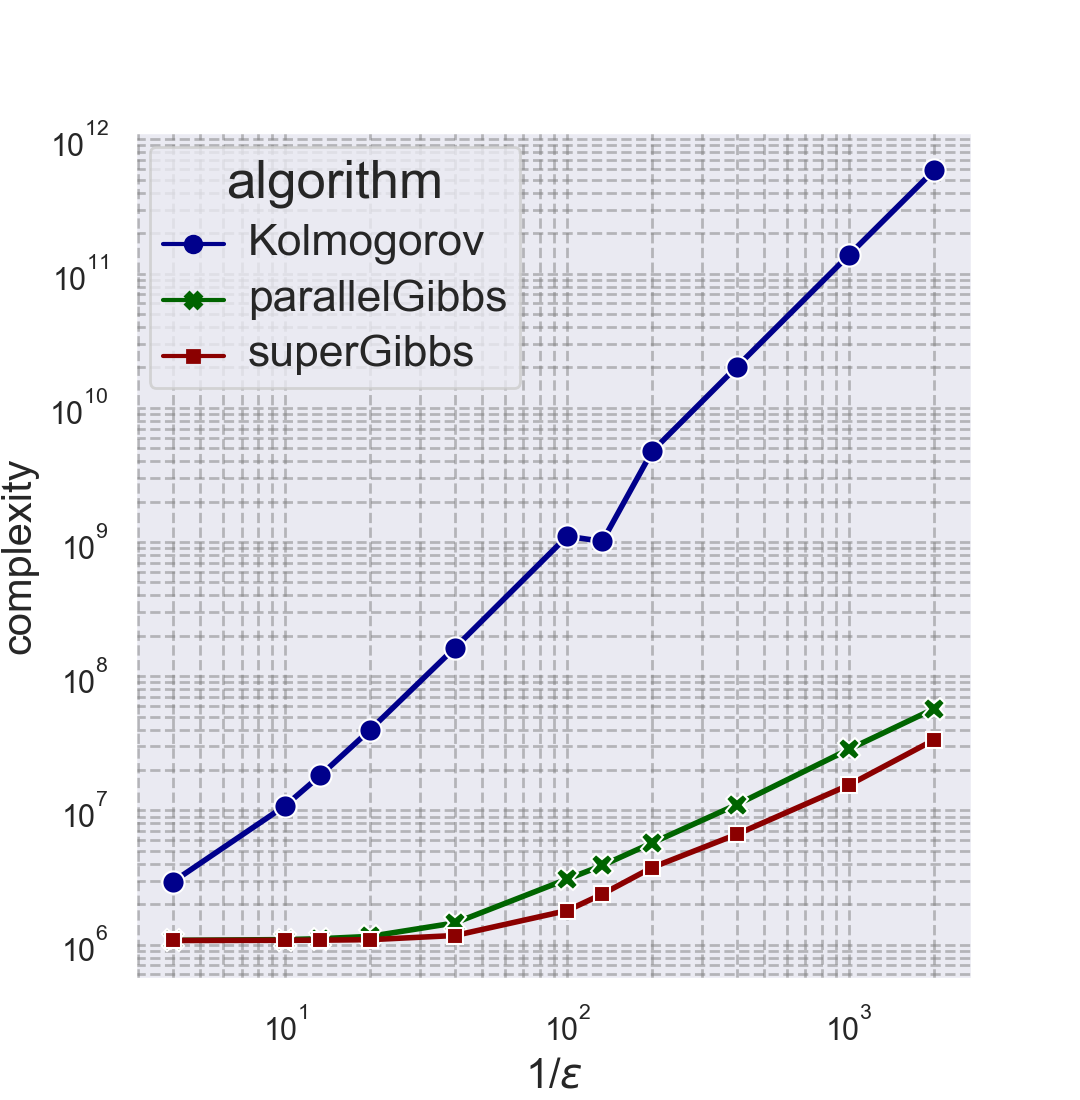}
    \caption{$\beta = .01$, 3$\times$3 lattice} \label{fig:1b}
  \end{subfigure}
  }%
  \hspace*{\fill}
  \scalebox{0.875}{
  \begin{subfigure}{0.3\textwidth}
      \centering
    \includegraphics[width=6cm,height=6cm]{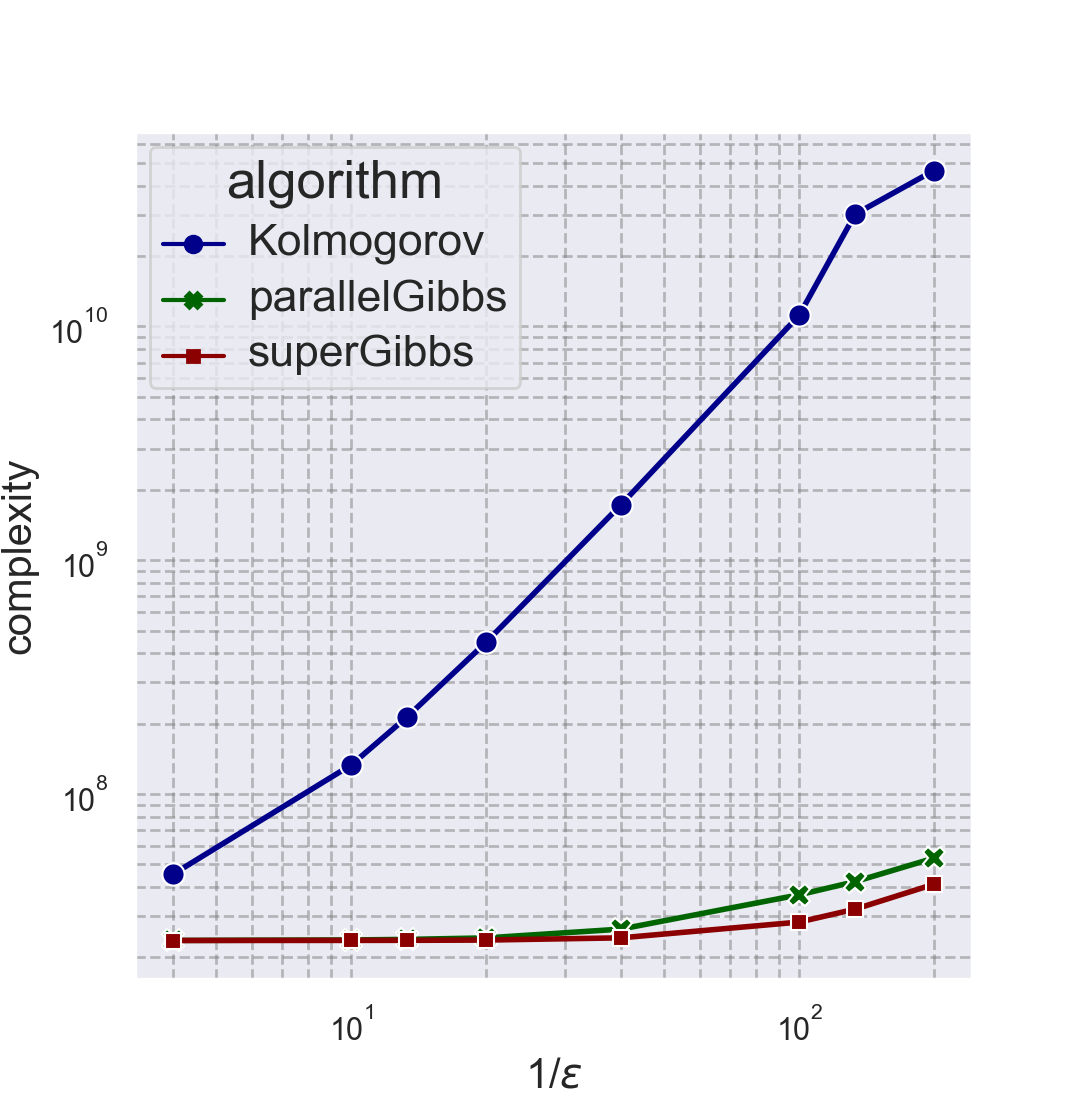}
    \caption{$\beta = .002$, 6$\times$6 lattice} \label{fig:1d}
  \end{subfigure}}%
  \hspace*{\fill}
  \scalebox{0.875}{
    \begin{subfigure}{0.33\textwidth}
    \includegraphics[width=6cm,height=6cm]{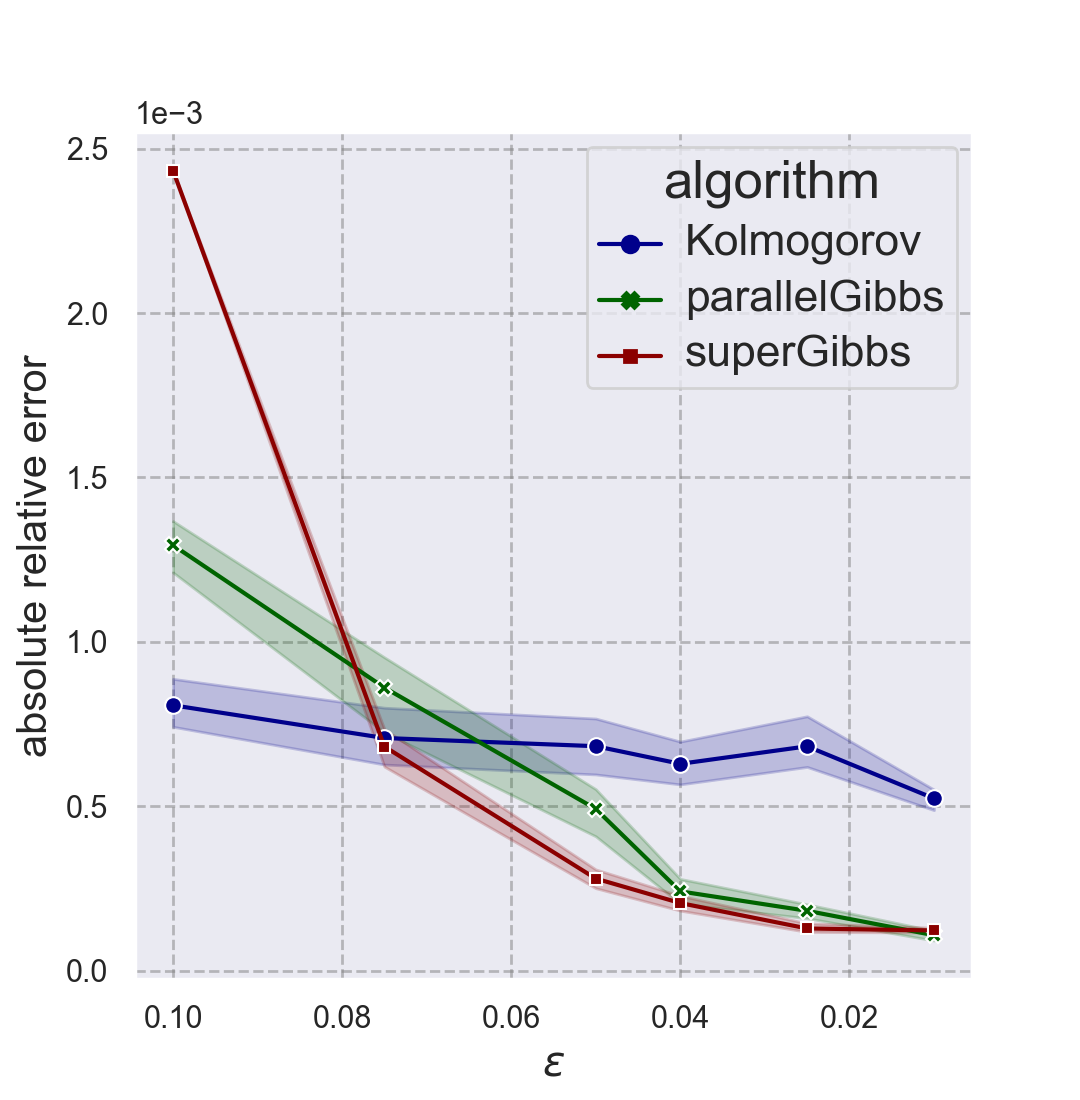}
    \caption{relative errors} \label{fig:error}
  \end{subfigure}%
  }
  \hspace*{\fill}

\caption{\footnotesize Comparison of sample complexity and precision $\frac{1}{\varepsilon}$ on Ising models. See also the A.6. of the supplementary material} \label{fig:ising2}
\end{figure*}


{\bf Results:} Our experiments demonstrate the practical advantages of our \emph{doubly adaptive method}, validating our theoretical analysis.   

(1) We first compare the complexity of our algorithms to Kolmogorov's algorithm. 
 Our experiments  show the superiority of both versions of our methods on different models and various sets of parameters. \Cref{fig:ising2}  demonstrates the superiority of our methods on the Ising model for various sets of parameters, and in \cref{fig:2aa,fig:2cc}  for the voting model,  when $\varepsilon$ is fixed and $\Zi{\beta}$ is varying (\cref{fig:2aa}), and  when $Z$ is fixed and $\varepsilon$ is varying (\cref{fig:2cc}). 
 All of these 
hold while the precision of our algorithms beats \cite{Kolmogorovgibbs} as $\varepsilon\rightarrow 0$ (\cref{fig:error})\cyrus{confusing}.

(2) To demonstrate the advantage of using the relative \emph{trace} variance, in contrast to the relative variance, we run both of our algorithms using a simpler mean estimator which only uses progressive sampling, and we compare the results. This is done by setting $T\gets 1$ in line~\ref{alg:rme:trel} of \proc.  In \Cref{fig:2a}, we show the effectiveness of \emph{trace averaging}, since both \superalgo\ and \paralgo\ beat their simplified versions ($T\gets 1$) after $\nicefrac{1}{\varepsilon}$ passes a certain threshold. This is consistent for different parameters of the voting model.

(3) Comparing the performance of \superalgo\ and \paralgo, we observe that in
all of our experiments \superalgo\ has better performance than \paralgo\cyrus{Always or eventually?}. In \cref{fig:2a}, 
we show the  trace variance term  \paralgo\ becomes dominant earlier as $\nicefrac{1}{\varepsilon}$ 
grows, 
thus it performs better in this perspective. 
This is consistent with our theoretical findings, because the ranges of estimators in \paralgo\ are smaller than the ranges 
used in \superalgo.

\begin{figure*}
\captionsetup[figure]{font=tiny}
\scalebox{0.87}{
  \begin{subfigure}{0.3\textwidth}
    \hspace{-0.1cm}\includegraphics[width=5.1cm,height=5cm]{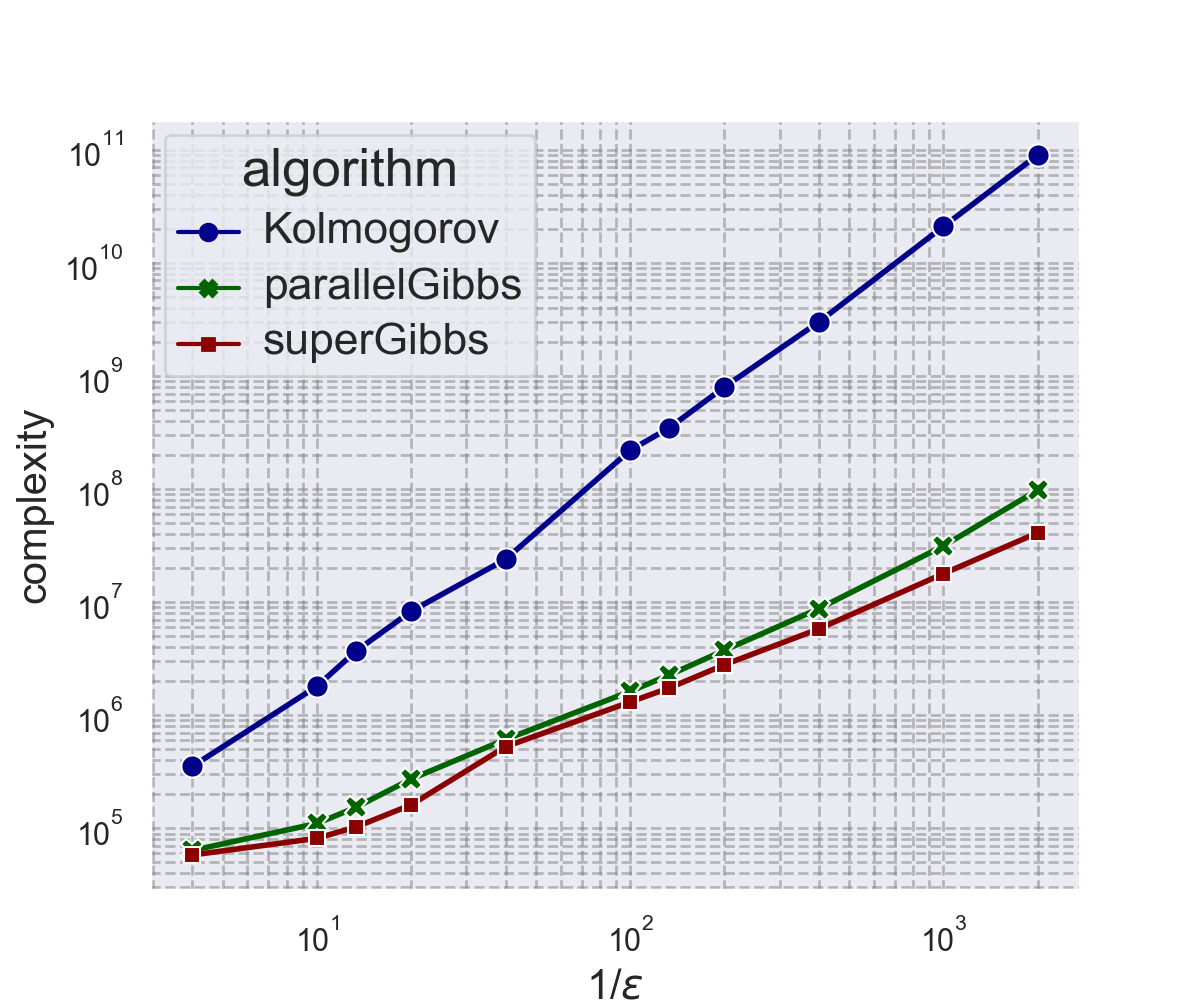}
    \caption{Complexity vs.\ $\varepsilon$ comparison against \tpa\ + PPE \cite{Kolmogorovgibbs}.} \label{fig:2cc}
  \end{subfigure}%
  }  \hspace*{\fill}
    \scalebox{0.84}{
  \begin{subfigure}{0.33\textwidth}
    \vspace{-0.0cm}
    \hspace{-1cm}\includegraphics[width=7.7cm,height=5cm,trim={1.5cm 0cm 0cm 0.25cm},clip]{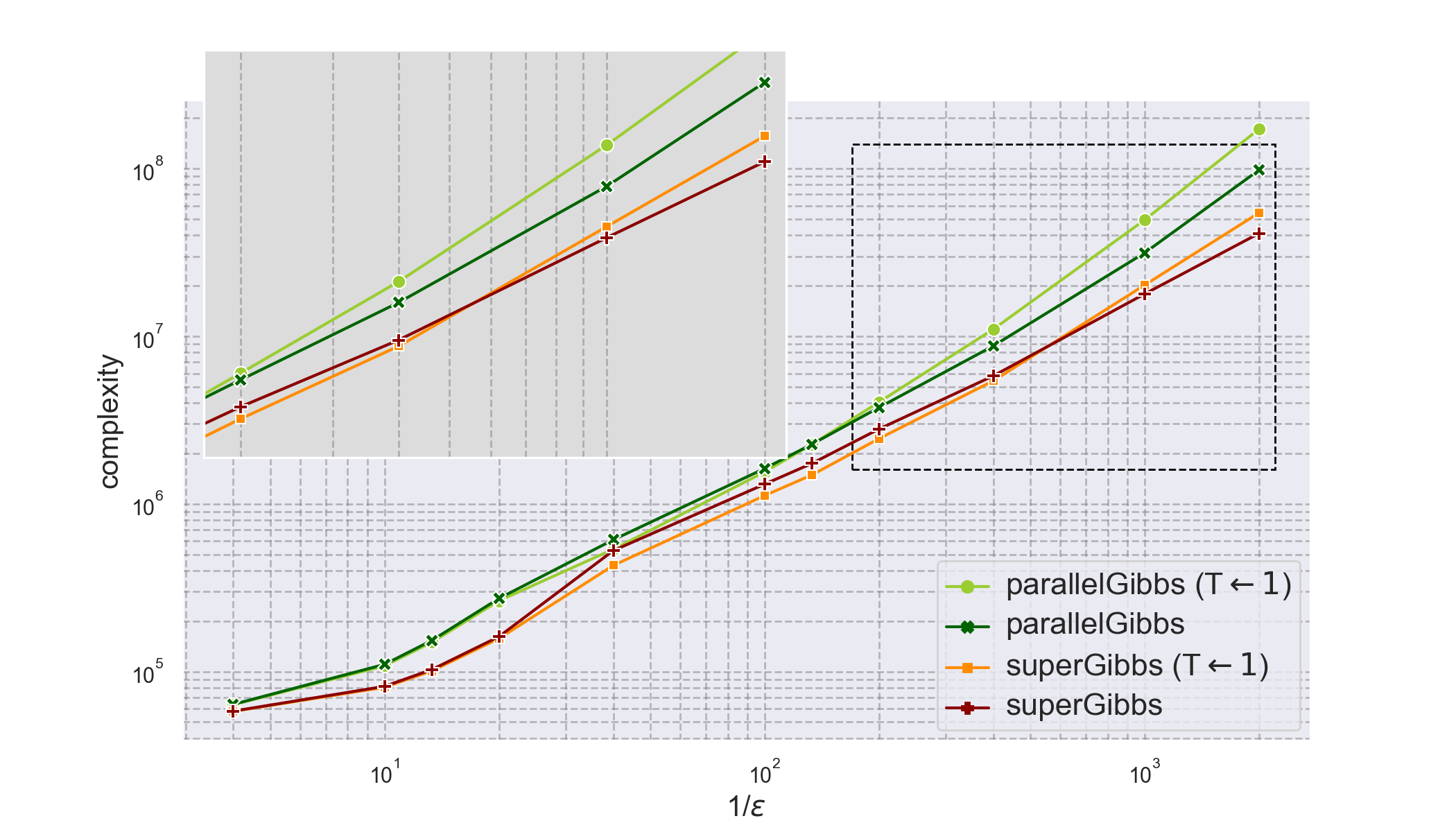}
    \vspace{-0.55cm}
    \caption{\small{Comparison of our algorithms and the effect of trace variance}.} \label{fig:2a}
    \vspace{-0.0cm}
  \end{subfigure} } 
  \hspace*{\fill}  \scalebox{0.86}{
   \begin{subfigure}{0.25\textwidth}
    \hspace{0.3cm}
    \includegraphics[width=5cm,height=5cm]{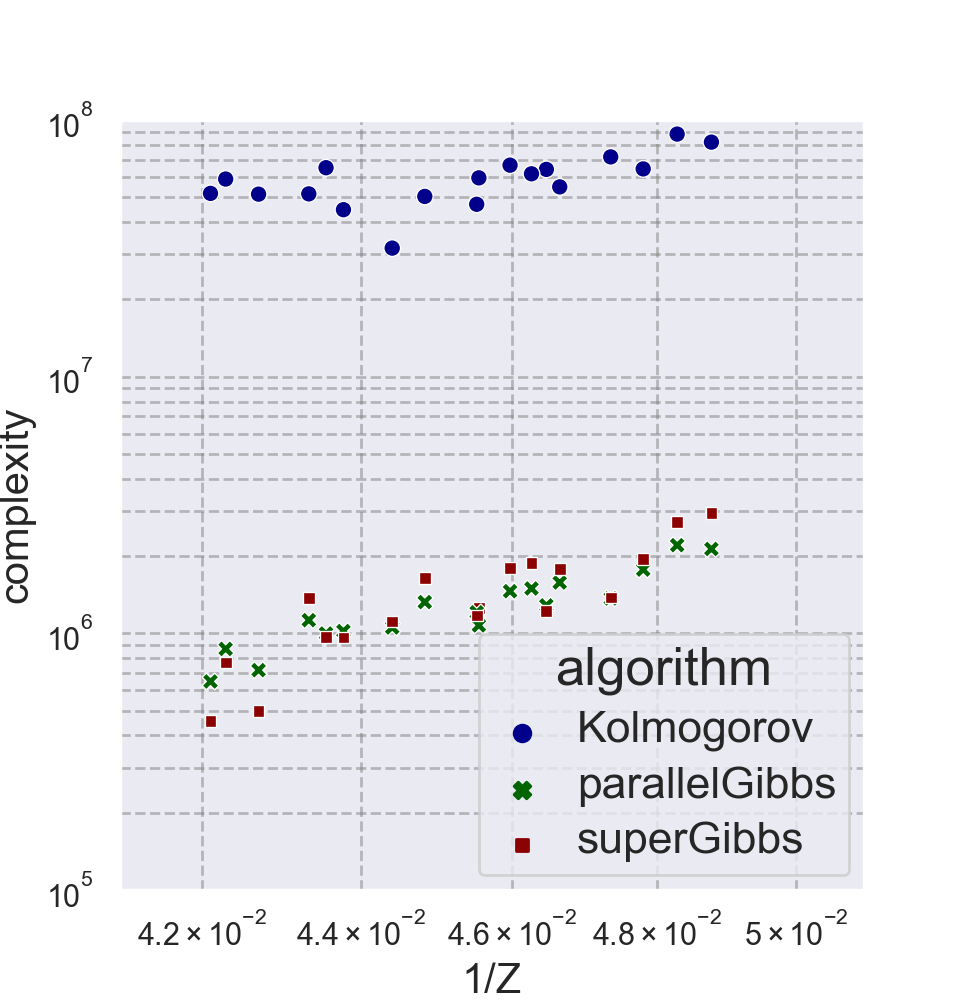}
    \vspace{-0.5cm}
    \caption{\small{}Complexity vs.\ $Z(\beta)$; 
    $ \varepsilon=0.025$ comparison against \tpa\ + PPE \cite{Kolmogorovgibbs}. } \label{fig:2aa}
  \end{subfigure}%
  }
  \hspace*{\fill}
 
\vspace{-0.1cm}
\caption{\small
Experiments on voting models. In (a) and (b) the parameters  are $\beta = 0.1, n=3, \omega = 0.9, 
    \omega_{T} = \langle 0.2, 0.5, 0.1 \rangle$ and $\omega_{F} = -\langle 0.8, 0.2, 0.9 \rangle$. In (c), we have $n=5$, and the weights and $\beta$ are picked randomly\cyrus{uniformly?} to generate models with various values of $Z(\beta)$.
}\label{fig:voting}
\end{figure*}

 \vspace{-0.2cm}
\section{Conclusions: advantages and limitations of proposed algorithms}

We develop a doubly-adaptive MCMC-based estimator for the partition function of Gibbs distributions, which resolves a major impediment of prior methods that use MCMC as a black-box sampler.
We show, both theoretically and experimentally, that our method requires substantially fewer MCMC steps than the state-of-the-art method.
The better performance is due to several factors, which all stem from the use of an \emph{adaptive MCMC mean estimator} instead of a standard "black-box" MCMC estimate. The complexity of the adaptive MCMC process
depends on the (smaller) \emph{trace}, rather than \emph{stationary}, relative variances, and on \emph{relaxation times} instead of \emph{mixing times}. It is also less sensitive to weak upper-bounds on mixing and relaxation times.

In particular, Kolmogorov's method  requires $\Theta(\nicefrac{\ell}{\varepsilon^2})$ approximately independent samples, where $\ell$ is the length of cooling schedule. This requires tight convergence (total variance distance of $O(\nicefrac{\varepsilon^{2}}{\ell})$ from stationary) for each sample, which adds a multiplicative $\ln \frac{\ell}{\varepsilon^{2}}$, with $\ell=\Theta(\ln Q \ln H_{\max})$, to its complexity (see column 3 of \cref{tab:comparisioncon} and \cite{Kolmogorovgibbs}, theorem 9).
In contrast, our doubly adaptive method \emph{only} depends on relaxation times, which do not depend on $\varepsilon$. 

{\bf Limitations.} While significantly improving the state of the art, our methods suffer from a several limitations. In \superalgo, the major limitation is the dependence  on the \emph{relative ranges} of $F$ and $G$, which can be large,  especially when the Hamiltonian range is large. Another issue is that the product chain's mixing time is dominated by $\ell \max\{\tau_i\}_{i=1}^\ell$, as opposed to $\sum_{i=1}^\ell \tau_i$.
While 
\paralgo\ circumvents these issues by estimating each factor of the telescoping product independently, it fails to beat \superalgo's efficiency in general, 
due both to the union bound and the higher-precision guarantees required for each subproblem.
Improving performance further will likely require 
new estimators with smaller ranges and relative trace variances. 

{\bf Statement of Broader Impact.}
While probabilistic graphical models as other machine learning methods that rely on MCMC estimations continue to grow in importance and popularity. But running the MCMC to theoretical convergence guarantees is often prohibitively expensive, while running it to \emph{apparent convergence} is methodologically unsound, particularly in the modern context, where public confidence in machine learning systems is continuously eroded by ethical, accuracy, and safety failures.
Our work attempts to bridge the gap between the definite, elegant and theoretically sound analytic methods, and efficiency-focused practical utility, as we seek to reduce \emph{proof-burden}, while maintaining theoretical guarantees of accuracy, with adaptive methods that bound efficiency in terms of (potentially unknown) convergence rate metrics and variances.

{\bf Acknowledgements.}
Shahrzad Haddadan is supported by NSF Award CCF-1740741. Cyrus Cousins and Eli Upfal are supported by NSF grant RI-1813444 and DARPA/AFRL grant FA8750. The authors are thankful to anonymous reviewers of NeurIPS 2021 for several valuable inputs. 

\vfill
\pagebreak[1]

\vfill
\pagebreak[4]

\vfill
\pagebreak[4]

\appendix

\section{Appendix}

\subsection{
Algorithms used in the literature} 
\subsubsection{The \tpa\ method \cite{HuberTPA2010,Kolmogorovgibbs}}\label{app:TPA}

We refer to Huber and Schott's algorithm as the original \tpa, and Kolmogorov's, which is used in our algorithms and referred to as \tpa$(k,d)$ in the main manuscript,  as the \tpa\ method.

\cyrus{Call these $\tpa(k)$ and $\tpa(k,d)$?}

\begin{algorithm}[ht]
\algrenewcommand\algorithmicindent{1.0em}
\begin{algorithmic}[1]
\State \textbf{output} a schedule $(\beta_1,\dots , \beta_l)$ of values in the interval $[\beta_{\min}, \beta_{\max}]$.
\State $\beta_0 \gets \beta_{\min}$
\For{$i=0:\infty$}
\State sample $X\sim \gibbs{\beta_i}$ draw $U\in [0,1]$ uniformly, $\beta_{i+1}=\beta_i- \log U/H(X)$ (or $+\infty$ if $H(X)=0$.)
\If{$\beta_{i+1}\notin [\beta_{\min},\beta_{\max}]$} Terminate 
\EndIf 
\EndFor
\end{algorithmic}
\caption{\textsc{The Original Tpa-method} \cite{HuberTPA2010}}\label{alg:tpahuber}
\end{algorithm}

\begin{algorithm}[ht]
\algrenewcommand\algorithmicindent{1.0em}
\begin{algorithmic}[1]
\State \textbf{input} integers $k$ and $d$
\State \textbf{output} a schedule $(\beta_0,\beta_1,\dots , \beta_l)$ of values in the interval $[\beta_{\min}, \beta_{\max}]$.

\For{$i=1:k$}
\State ${\cal B}_i\gets \textsc{The Original TPA-method}()$.
\State let ${\cal B}\gets {\cal B}\cup {\cal B}_i$
\EndFor 
\State sort $\cal B$, keep one sample uniformly from the initial $d$ elements, and keep every $d$th successive value in the remaining sequence. 
\State add $\beta_{\min}$ and $\beta_{\max}$ to $\cal B$ 
\Return ${\cal B}$ 
\end{algorithmic}
\caption{\textsc{ Tpa-Method} \cite{Kolmogorovgibbs}}\label{alg:tpakol}
\end{algorithm}

\subsubsection{Single site Gibbs sampler (Glauber dynamics chain)}\label{App:gibbschain}

Consider  $\beta$ and $H$ defined as above. 
Let $X=(X_1,X_2,\dots, X_n)$ be the set of all variables in the Gibbs distribution with inverse temperature $\beta$ and Hamiltonian $H$, thus, the domain of $H$ is $\Omega=\Omega_1\times \Omega_2\times \dots \Omega_n$, and each $\Omega_i$ is the range of random variable $X_i$. At each time step $t$, assume the current state is $x^{(t)}=(x_1,x_2,\dots ,x_n)$. Take $i\sim 1, \dots, n$ uniformly at random. Sample $y$ from the following distribution: 

\begin{equation}
     \gibbs{\beta}(y\vert x^{(t)}_{-i})=\frac{\exp(-\beta H(x^{(t)};x_i\gets y))}{\sum_{\omega\in\Omega_i}\exp(-\beta H(x^{(t)};x_i\gets \omega))} \enspace,
\end{equation}
where for an arbitrary $\omega \in \Omega_i$ we define $(x^{(t)};x_i\gets \omega)$ be the vector in which all the elements except the $i$th element are equal to $x_i$ and the $i$th element is replaced with $\omega$.

In other words, for any arbitrary vectors $x^{(t)}$ and  $x^{(t+1)}$, the transition probability is: 

$$
\gibbschain{H}{\beta}(x^{(t)}, x^{(t+1)})=\begin{cases}
 (\nicefrac{1}{n})\gibbs{\beta}(y\vert x^{(t)}_{-i}), &  \exists y,i \text{ such that }x_i\neq y \text{ and } x^{(t+1)}= (x^{(t)};x_i\gets y),\\
 \sum_{i=1}^n (\nicefrac{1}{n})\gibbs{\beta}( x_i\vert x_{-i}^{(t)}) & \text{if } x^{(t)}=x^{(t+1)},\\
 
 0 &\text{otherwise~.}

\end{cases}
$$

\subsection{Missing proofs: TPA  and relative trace variance properties}

\cyrus{sort proofs; in weird order}

\cyrus{restate here?}
    
    \begin{lemma}\label{lem:intervalbd} Let $z(\beta)\doteq\log\left(\Zi{\beta}\right)$, $d$ and $k$  the parameters of the \tpa\ method, and $\beta_i$ and $\beta_{i+1}$ two consecutive points generated by $\tpa(k,d)$, we have:
    \vspace{-0.2cm}
    \begin{enumerate}[wide, labelwidth=0pt, labelindent=0pt]\setlength{\itemsep}{0pt}\setlength{\parskip}{0pt}
        \item   For any $\varepsilon\geq 0$, we have $\mathbb{P}(z(\beta_j)-z(\beta_{j+1})\leq \varepsilon)\geq (1-\exp(-\varepsilon k/d))^d\simeq 1-d\exp(-\varepsilon k/d) $~,
        \item  For any $\varepsilon\geq 0$, $\mathbb{P}\left(\Delta_i\geq \nicefrac{\varepsilon}{\mathbb{E}[H(x)]}\right)\leq  d \exp(-\varepsilon k/d),$ where the expectation of $H(x)$ is taken with respect to distribution $x\sim \gibbs{\beta_{i+1}}$.
    \end{enumerate}
    \end{lemma}

\begin{proof}[Proof of \cref{lem:intervalbd}]
Note that $\tpa(k,d)$ of \cite{Kolmogorovgibbs} consists of $k$ parallel runs of the original \tpa\ of \cite{HuberTPA2010} and outputting a sub-sequence of elements which are $d$ apart.

\cyrus{Fix notation.}

Let  $(b_i)$ be the sequence generated by $k$ parallel copies of the original \tpa, thus $\Delta_j=\beta_{j+1}-\beta_{j}=b_{j+d}-b_{j} $.

We first show item 1 by bounding $\mathbb{P}\left(z(b_{j})-z(b_{j+d})\geq \varepsilon\right)$,
and using 
$$
     \mathbb{P}(b_{j+d}-b_{j}<\varepsilon)  \geq \prod_{i=1}^d\mathbb{P}(b_{j+i}-b_{j+i-1} < \varepsilon/d) ~.
$$ 

With the definition of the \cyrus{Poisson point process?}PPP, and using \cite{KolmoHarrisGibbs} \cyrus{which result?} we have $z{(b_{i})}-z{(b_{i+1})}$ follows the exponential distribution with mean $1/k$,  thus
$
    \mathbb{P}(z{(b_{i})}-z{(b_{i+1})}\geq\varepsilon/d) = \exp(-\varepsilon k/d )~.
$ Therefore, 

$$
     \mathbb{P}(z(b_{j+d})-z(b_{j})<\varepsilon)  \geq \prod_{i=1}^d\mathbb{P}\left(z(b_{j+i})-z(b_{j+i-1}\right) < \varepsilon/d) =
     (1-\exp(-\varepsilon k/d))^d~.
$$

To see item 2 of the Lemma let $z'(\beta)$ be the derivative of $z(\cdot)$\cyrus{Do we define $z(\cdot)$?} with respect to $\beta$, which is $z'(\beta) = \sum_{x\in \Omega} -H(x) \exp(-\beta H(x))/Z(\beta)$, thus $z'(\beta)\leq 0$. Using the Cauchy–Schwarz inequality we have $z''(\beta) = (\sum_{x\in \Omega} H^2(x) \exp({-\beta H(x))} \sum_{x\in \Omega} \exp({-\beta H(x))} -
(\sum_{x\in \Omega} -H(x) \exp({-\beta H(x)})^2) /Z^2(\beta) \geq 0$.
 Therefore,
$$
    z'(\beta_i) < \frac{z(\beta_{i+1})-z(\beta_i)}{\beta_{i+1}-\beta_i} < z'(\beta_{i+1}), 
$$
Thus,
$
    \beta_{i+1}-\beta_{i} < \frac{z(\beta_{i})-z(\beta_{i+1})}{-z'(\beta_{i})} ~$. Note that    $-z'(\beta_i)=\mathbb{E}[H(x)], x\sim \gibbs{\beta_{i}}$.
    Therefore,  we have:

\begin{align*}
  \mathbb{P}\left(\Delta_i\leq \frac{\epsilon}{ \mathbb{E}[H]}\right)
  &\geq \mathbb{P}\left(\frac{z(\beta_i)-z(\beta_{i+1})}{-z'(\beta_{i+1})} \leq \frac{\epsilon}{ \mathbb{E}[H]}\right) & \\
  &=\mathbb{P}\left({z(\beta_i)-z(\beta_{i+1})} \leq {\epsilon}\right) & \\
  &\geq   (1-\exp(-\epsilon k/d))^d \\
\end{align*}

Thus $\mathbb{P}\left(\Delta_i\geq \frac{\epsilon}{ \mathbb{E}[H]}\right)\geq 1- (1-\exp(-\epsilon k/d))^d \approx d \exp(-\epsilon k/d)) \enspace.$
\end{proof}

\begin{proof}[Proof of Lemma 2.1.]

Note that by Thm 3.1. of \cite{paulin2015} we have, $\mathbb{E}[(\bar{f}(\vec{X}_{1:\tau})-\mathbb{E}(f))^2]\leq \frac{2\TRel}{\tau}\mathbb{V}[f]$. Dividing both sides by $\left(\mathbb{E}(f)\right)^2$   we get the second part of the premise. The first part concludes from setting $\tau=\TRel$.

\end{proof}

\cyrus{Sort order of proofs}

\subsection{\proc}
{\bf \proc\ in summary}
To employ progressive sampling, we  start by a small sample size and calculate the \emph{empirical estimation} of the variance at each iteration.
We estimate an upper bound on the trace variance based on its empirical estimation, and using that we check a termination condition.

Our variance estimator is what  Cousins et al.\ introduced, and is based on running two independent chains. 
Each sample is obtained by 
taking a trace of length $T$ (given upper-bound on relaxation time) and taking the average over all observed values on that trace. 
Thus, \emph{half the square difference} of the averages on the two chains is an \emph{unbiased estimate} of the \emph{trace variance}.

Before showing the result, we state two key theorems from the literature, which describe how our tail bounds work.
\begin{theorem}[Hoeffding-Type Bounds for Mixing Processes, ({see Thm.~2.1 of \cite{fan2018hoeffding}})]\label{thm:hoeffding}
For any $\delta \in (0, 1)$, 
we have
\begin{equation}
\label{eq:Hoeffding}
\Prob\!\left( \lvert \hat{\mu} - \mu \rvert \geq \sqrt{\frac{2(1+\EigTwo)(\mathsmaller{\frac{\frange^{2}}{4}})\ln(\frac{2}{\delta})}{(1-\lambda)m}} \right) \leq \delta \enspace.
\end{equation}

This implies sample complexity

\[
m_{H}(\EigTwo, \frange, \varepsilon, \delta)
  = \frac{1+\lambda}{1-\lambda} \ln(\mathsmaller{\frac{2}{\delta}}) \frac{\frange^{2}}{2\varepsilon^2}
  \in \Theta\Bigl(\TRel\ln(\mathsmaller{\frac{1}{\delta}})\frac{\frange^{2}}{\varepsilon^{2}}\Bigr)
  \enspace.
\]
\end{theorem}
\begin{theorem}[Bernstein-Type Bound for Mixing Process {\cite[Thm.~1.2]{jiang2018bernstein}}]
For any $\delta \in (0, 1)$, we have
\label{thm:bernstein}

\begin{equation}\label{eq:Bernstein}
 \Prob\left(\vert \hat{\mu}-\mu\vert\geq \frac{10\frange\ln(\frac{2}{\delta})}{(1-\lambda)m} + \sqrt{\frac{2(1+\EigTwo)\SVar\ln(\frac{2}{\delta})}{(1-\lambda)m}} \right) \leq \delta \enspace.
\end{equation}
\noindent
This implies sample complexity
\vspace{-0.2cm}
\[
m_{B}(\EigTwo, \frange, v, \varepsilon, \delta)
  = \frac{2}{1-\lambda}\ln(\mathsmaller{\frac{2}{\delta}})\Bigl(\frac{5\frange}{\varepsilon} + \frac{(1+\EigTwo)\SVar}{\varepsilon^2}\Bigr)
  \in \Theta\Bigl(\TRel\ln(\mathsmaller{\frac{1}{\delta}})\Bigl(\frac{\frange}{\varepsilon} + \frac{\SVar}{\varepsilon^{2}}\Bigr)\Bigr) \enspace.
\]
\end{theorem}

We now show the main result.
\begin{proof}[Proof of Theorem 2.2]

Suppose confidence interval $[a, b]$.
The interval endpoints, multiplicative error $\varepsilon_{\times}$, and additive error $\varepsilon_{+}$ are related as
$
2\varepsilon_+ = a\frac{1+\varepsilon_{\times}}{1-\varepsilon_{\times}}-a=a\frac{2\varepsilon_{\times}}{1-\varepsilon_{\times}}$, 
depicted graphically below.

\begin{tikzpicture}[xscale=4.5,yscale=0.5]

\node[thick,circle,draw=black,fill=gray,inner sep=1.25pt] (zero) at (0, 0) {};
\node[thick,circle,draw=black,fill=gray,inner sep=1.25pt] (a) at (1, 0) {};
\node[thick,circle,draw=black,fill=gray,inner sep=1.25pt] (mu) at (2, 0) {};
\node[thick,circle,draw=black,fill=gray,inner sep=1.25pt] (b) at (3, 0) {};

\node[below of={zero},yshift=0.5cm] {$0$};
\node[below of={a},yshift=0.5cm] {$a$};
\node[below of={mu},yshift=0.5cm] {$\mu$};
\node[below of={b},yshift=0.5cm] {$b$};

\draw (zero) -- (a) -- (b);

\node[thick,circle,draw=blue,fill=gray,inner sep=1.25pt] (ll) at (1.0, 0.1) {};
\node[thick,circle,draw=purple,dotted,fill=gray,inner sep=1.25pt] (cl) at (1.1, 0.1) {};
\node[thick,circle,draw=red,fill=gray,inner sep=1.25pt] (rl) at (1.21, 0.1) {};

\draw (ll) -- (cl) -- (rl);

\draw [decorate,decoration={brace,amplitude=10pt,raise=4pt},yshift=0pt] (ll) -- (rl) node [above,black,midway,yshift=1.0cm] {\begin{tabular}{c} \scriptsize{Worst Case: }
$\varepsilon_{+} = \frac{a\varepsilon_{\times}}{1-\varepsilon_{\times}}$
\\ \end{tabular}};

\node[thick,circle,draw=blue,fill=gray,inner sep=1.25pt] (lc) at (1.8, 0.1) {};
\node[thick,circle,draw=purple,dotted,fill=gray,inner sep=1.25pt] (cc) at (2, 0.1) {};
\node[thick,circle,draw=red,fill=gray,inner sep=1.25pt] (rc) at (2.2, 0.1) {};

\draw (lc) -- (cc) -- (rc);

\draw [decorate,decoration={brace,amplitude=10pt,raise=4pt},yshift=0pt] (lc) -- (rc) node [above,black,midway,yshift=0.5 cm] {\begin{tabular}{c} \scriptsize{Arbitrary Case:} 
$\varepsilon_+ = \mu \varepsilon_{\times}$ \\
\end{tabular}};

\node[thick,circle,draw=blue,fill=gray,inner sep=1.25pt] (lh) at (2.45454545, 0.1) {};
\node[thick,circle,draw=purple,dotted,fill=gray,inner sep=1.25pt] (ch) at (2.72727272, 0.1) {};
\node[thick,circle,draw=red,fill=gray,inner sep=1.25pt] (rh) at (3.0, 0.1) {};

\draw (lh) -- (ch) -- (rh);

\draw [decorate,decoration={brace,amplitude=10pt,raise=4pt},yshift=0pt] (lh) -- (rh) node [above,black,midway,yshift=1.0cm] {\begin{tabular}{c} \scriptsize{Best Case:} 
$\varepsilon_+  = \frac{b\varepsilon_{\times}}{1+\varepsilon_{\times}}$ \\
\end{tabular}};

\end{tikzpicture}

We derive a geometric progressive sampling schedule such that the algorithm draws sample sizes, ranging between optimistic and pessimistic (over unknown variance and mean) upper and lower bounds on the sufficient sample size.

Using the Markov chain Bennett inequality \cite{jiang2018bernstein}, the best-case  complexity, assuming maximal expectation, and minimal variance, is
\begin{align*}
m^{\downarrow} &\geq m_{B}(\EigTwoBound, \frange, 0, \varepsilon_{+}, \frac{2\delta}{3I})\\
&\geq \frac{(1+\EigTwoBound)\frange \ln \frac{3I}{\delta}}{(1-\EigTwoBound)\varepsilon_{+}} = \frac{(1+\EigTwoBound)\frange \ln \frac{3I}{\delta} (1 + \varepsilon_{\times})}{b(1-\EigTwoBound)\varepsilon_{\times}} \enspace.
\end{align*}

\vspace{-0.2 cm}

The worst-case  complexity, then assuming minimal expectation, and maximal variance, is 
\begin{align*}
m^{\uparrow} &\geq m_{H}(\EigTwoBound, \frange, \varepsilon_{+}, \frac{2\delta}{3I}) \\
&\geq \frac{(1+\EigTwoBound)\frange^{2} \ln \frac{3I}{\delta}}{2(1-\EigTwoBound)\varepsilon_{+}^{2}} = \frac{(1+\EigTwoBound)\frange^{2} \ln \frac{3I}{\delta}(1 - \varepsilon_{\times})^{2} }{2(1-\EigTwoBound)a^{2}  \varepsilon_{\times}^{2}} \enspace,
\end{align*}
\cyrus{With Huber bound on variance?}

via the Markov chain Hoeffding's inequality \cite{leonhoeffdingbound2004}.

Consequently, a doubling schedule requires
$
I = \left \lfloor \log_{2}\left( \frac{m^{\uparrow}}{m^{\downarrow}} \right) \right \rfloor = \left \lfloor \log_{2}\left( \frac{b\frange}{2 a^{2}} \cdot \frac{ (1-\varepsilon_{\times})^{2}}{ (1 + \varepsilon_{\times})\varepsilon_{\times}} \right) \right \rfloor
$ steps.

All tail bounds on variances and means are hold simultanously with probability at least $1 - \delta$ (by union bound), and the doubling schedule never overshoots the sufficient sample size by more than a constant factor, which yields the stated guarantees.

The proof consists of two parts, in both we make derive our new bounds by writing an  $\varepsilon_\times$-multiplicative approximation in terms of an $\varepsilon_{+}$-additive approximation.

In the worst-case, we \emph{underestimate} the true mean $\mu$ by a factor $(1 - \varepsilon_{\times})$, and thus require a radius $\varepsilon_{+} = \varepsilon_{\times}(1 - \varepsilon_{\times})\mu$ additive confidence interval.

We first show the \emph{correctness guarantee}.

Observe that the sampling schedule is selected such that the final iteration $I$ of the algorithm will draw a sufficiently large sample (size $m^{\uparrow}$) such that the Hoeffding inequality will yield such a confidence interval, even for worst-case (minimal) $\mu$.
Now observe that over the course of the algorithm, in each iteration, 3 tail bounds are applied; one to upper-bound the variance, and then two to upper and lower bound the mean in terms of the variance bound) as in \cite{Dynamite}.
By union bound, all $3I$ tail-bounds hold simultaneously with probability at least $1-\delta$, thus when the algorithm terminates, it produces a correct answer with at least said probability.

\bigskip

We now show the \emph{efficiency guarantee}.
Suppose we get $\hat{\mu}$ from \proc, by guarantee of correctness of the algorithm, we have a lower bound on $\hat{\mu}$, $\hat{\mu} \geq \mu (1-\varepsilon_{\times})$ with probability at least $1-\delta$. 

Furthermore, we have 
$
\varepsilon_{+} = \mu \varepsilon_{\times} 
$ and 
$
\ITRelVar = (\RelITRelVar{}{\TRel} - 1)\times \hat{\mu}^2 \geq 
(\RelITRelVar{}{\TRel} - 1) \mu ^2 (1-\varepsilon_{\times})^2
$.
For this ${\varepsilon_{+}}$, we have via the Bernstein inequality that
\[
m^{*} \in 
\mathcal{O}\left( \log\left(\frac{\log(\frange/(\mu\varepsilon_{\times}))}{\delta}\right)
\left( \frac{\frange/\mu}{(1 - \EigTwoBound)\varepsilon_{\times}} + \frac{\TRel(\RelITRelVar{}{\TRel} - 1) }{\varepsilon_{\times}^{2}}\right)\right)
\]
would be a sufficient sample size if (1) the algorithm were to draw a sample of this size, and (2) we were to use the \emph{true trace variance} instead of the \emph{estimated upper-bound on trace variance}.

Fortunately, correcting for (1) adds a constant factor to the sample complexity, as the first sample size $\alpha$ is selected to be twice the minimal sufficient sample size $m^{\downarrow}$ (i.e., the sample size such that no smaller sample size would be sufficient), and at each iteration the sample size selected is double the previous (line~\ref{alg:ss}).
In other words, this geometric grid will never overshoot any sample size by more than a factor 2.

Resolving (2) is a bit more subtle, but 
we now show that there is no asymptotic change in replacing the variance with the estimated variance upper bound (w.h.p.).
First, note that the Bernstein bound is \emph{bidirectional}, so it can just as well be used to upper-bound empirical variance with true variance as to upper-bound true variance with empirical variance.
We bound true variance in terms of empirical variance on line~\ref{alg:2chain-var}, and note that here we have
\[
v \leq u \in \hat{v} + \mathcal{O}\left(\frac{\frange^{2} \ln \frac{I}{\delta}}{m} + \sqrt{\frac{\frange^{2}\hat{v} \ln \frac{I}{\delta}}{m}}\right) \enspace.
\]
Fortunately, the latter terms are negligible, as in line~\ref{alg:bern}, we bound
\begin{align*}
\varepsilon_{+} &\in 
 \mathcal{O}\left(\frac{\frange \ln \frac{I}{\delta}}{m} + \sqrt{\frac{u \ln \frac{I}{\delta}}{m}}\right) \\
 &= \mathcal{O}\left(\frac{\frange \ln \frac{I}{\delta}}{m} + \sqrt{\frac{ \biggl(\hat{v} + \mathcal{O}\biggl(\mathsmaller{\frac{\frange \ln \frac{I}{\delta}}{m}} + \mathsmaller{\sqrt{\frac{\hat{v} \ln \frac{I}{\delta}}{m}}}\biggr)\biggr) \ln \frac{I}{\delta}}{m}}\right)  & \\
 &= \mathcal{O}\left(\frac{\frange \ln \frac{I}{\delta}}{m} + \sqrt{\frac{ \biggl(v + \mathcal{O}\biggl(\mathsmaller{\frac{\frange \ln \frac{I}{\delta}}{m}} + \mathsmaller{\sqrt{\frac{v \ln \frac{I}{\delta}}{m}}}\biggr) + \mathcal{O}\biggl(\mathsmaller{\frac{\frange \ln \frac{I}{\delta}}{m}} + \mathsmaller{\sqrt{\frac{\hat{v} \ln \frac{I}{\delta}}{m}}}\biggr)\biggr) \ln \frac{I}{\delta}}{m}}\right) & \text{(w.h.p.)} \\
 &= \mathcal{O}\left(\frac{\frange \ln \frac{I}{\delta}}{m} + \sqrt{\frac{{v} \ln \frac{I}{\delta}}{m}}\right) \enspace. & \text{(w.h.p.)} \\
\end{align*}

Putting these together, we thus have that, w.h.p., sample consumption is bounded as
\[
\hat{m} \in 2\mathcal{O}(m^{*}) = \mathcal{O}\left( \log\left(\frac{\log(\frange/(\mu\varepsilon_{\times}))}{\delta}\right)
\left( \frac{\frange/\mu}{(1 - \EigTwoBound)\varepsilon_{\times}} + \frac{\TRel(\RelITRelVar{}{\TRel} - 1) }{\varepsilon_{\times}^{2}}\right)\right) \enspace.
\]

To conclude, we need only relate $T(\RelITRelVar{}{T} - 1)$ and $\TRel(\RelITRelVar{}{\TRel} - 1)$.
Letting $T$ as in line~\ref{alg:tbound}, note that since $T \geq \TRel$, it holds that $T(\RelITRelVar{}{T} - 1) \geq \TRel(\RelITRelVar{}{\TRel} - 1)$, by the trace variance inequalities, which yields the result.
\end{proof}

\subsection{Missing proofs from analysis of \superalgo}
\begin{proof}[Proof of thm 2.4]
Follows immediately from thm. 2.2 and plugging in the values for paired product estimators and the product chain.  
\end{proof}

\begin{proof}[Full Proof of Lemma 2.8]
Let $\bar{\beta}_{i,i+1} \doteq \frac{\beta_i+\beta_{i+1}}{2}$, we have
${\mu_i}=  \frac{Z(\bar{\beta}_{i,i+1})}{Z(\beta_{i})} $ and 
${\nu_i}= \frac{Z(\bar{\beta}_{i,i+1})}{Z(\beta_{i+1})}$. Thus we have $ \nu = \frac{ \prod_{i=1}^{\ell-1} Z(\bar{\beta}_{i,i+1})}{ \prod_{i=1}^{\ell-1} Z(\beta_{i+1})} > 1$,  
$\mu = \frac{ \prod_{i=1}^{\ell-1} Z(\bar{\beta}_{i,i+1})}{ \prod_{i=1}^{\ell-1} Z(\beta_{i})} < 1$.

Note that $\nu = \mu \frac{Z(\beta_{0})}{Z(\beta_{\max})}$, thus we proceed by bounding $\mu$.

\begin{align*}
\log \prod_{i=1}^{\ell-1} Z(\bar{\beta}_{i,i+1}) &= \sum_{i=1}^{\ell-1} z(\bar{\beta}_{i,i+1})& \textsc{Taking $\log$}\\
& \geq  \sum_{i=1}^{\ell-1} z(\beta_{i}) - \frac{\Delta_{i}}{2} \Expect_{x \sim \pi_{\beta_{i}}}[H(x)]& \textsc{Taylor expansion \& that $\frac{\partial^{2}}{\partial \beta^{2}} z(\beta) > 0$}\\
\end{align*}
Thus, by taking exponents we get: 
\begin{align*}
\prod_{i=1}^{\ell-1} Z(\bar{\beta}_{i,i+1}) &\geq \exp \left( \sum_{i=1}^{\ell-1} z(\beta_{i}) - \frac{\Delta_{i}}{2} \Expect_{x \sim \pi_{\bar{\beta}_{i,i+1}}}[H(x)] \right)\\
&\geq \left( \prod_{i=1}^{\ell-1} Z(\beta_{i}) \right) \exp \left( -\sum_{i=1}^{\ell-1} \frac{\Delta_{i}}{2} \Expect_{x \sim \pi_{\beta_{i}}}[H(x)] \right)\\
\end{align*}

Therefore, $\mu= \frac{ \prod_{i=1}^{\ell-1} Z(\bar{\beta}_{i,i+1})}{ \prod_{i=1}^{\ell-1} Z(\beta_{i})} \geq \exp \left( -\sum_{i=1}^{\ell-1} \frac{\Delta_{i}}{2} \Expect_{x \sim \pi_{\beta_{i}}}[H(x)] \right)$. Using this form, we now employ the fundamental theorem of calculus to prove the premise: 

Let $\Delta_{\max}\dot{=}\max_{i}\Delta_{i}$.
\begin{align*}
\mu &\geq \exp \left( -\sum_{i=1}^{\ell-1} \frac{\Delta_{i}}{2} \Expect_{x \sim \pi_{\beta_{i}}}[H(x)] \right) & \\
 &= \exp \left( -\sum_{i=1}^{\ell-1} \frac{\Delta_{i}}{2} \Expect_{x \sim \pi_{\beta_{i}}}[H(x)] \right) & \\
 &\geq \exp \left(\frac{1}{2} \int_{\beta_{\min}-\Delta_{\max}}^{\beta_{\max}-\Delta_{\max}} -\Expect_{x \sim \pi_{\beta}}[H(x)] \, \mathrm{d} \beta \right) & \textsc{Increasing Integrand} \\
 &= \exp\left(\frac{1}{2} \bigl(z(\beta_{\max}-\Delta_{\max}) - z(\beta_{\min}-\Delta_{\max})\bigr)\right) & \textsc{FTOC and that $z'(\beta)=\mathbb{E}_{x\sim \gibbs{\beta}H}$} \\
 &\geq  \exp\left(\frac{1}{2} \bigl(z(\beta_{\max}) - z(\beta_{\min}-\Delta_{\max})\bigr)\right) & \textsc{$z$ is Decreasing} \\
 &=  \exp\left(\frac{1}{2} \bigl(z(\beta_{\max}) - z(\beta_{\min})+z(\beta_{\min}) -z(\beta_{\min}{-}\Delta_{\max})\bigr)\right) & \\
 &\geq Q^{-\frac{1}{2}} \sqrt{\frac{Z(\beta_{\min})}{Z(\beta_{\min}-\Delta_{\max})}}~.  \\
\end{align*}

\cyrus{Note that $\frac{Z(\beta_{\min})}{Z(\beta_{\min}-\Delta_{\max})} \geq \exp(-\Delta_{\max} H_{\max})$ for $\beta_{\min} = 0$.  General case, is it small?}

From the above we also conclude that $\nu\geq Q^{1/2}\sqrt{\frac{Z(\beta_{\min})}{Z(\beta_{\min}-\Delta_{\max})}}$.
Note that  ${\rm Range}(f)=\exp( -\frac{\Delta}{2} H_{\min})-\exp( -\frac{\Delta}{2} H_{\max})\leq \sqrt{\exp(-\Delta H_{\min})}$ and  ${\rm Range}(g)=\exp( \frac{\Delta}{2} H_{\max})-\exp( \frac{\Delta}{2} H_{\min})\leq \sqrt{\exp(\Delta H_{\max})}$~. Thus the lemma is concluded. 

\end{proof}

\begin{proof}[Proof of 
Corollary 2.6]
The corollary follows from thm 2.2 plugging in $R$ from lemma 2.5 and setting $\tau_{\rm prx}=\ell\max_{i=1}^\ell \tau_i$ (see, e.g.,  \cite{levin2017markov}).
\end{proof}

\subsection{Analysis of \paralgo}

Let 
 $(\beta_0,\beta_1,\dots \beta_l)$ be a cooling schedule generated by \tpa$(k,d)$, where $k$ and $d$ are chosen as in \cite{Kolmogorovgibbs}. For each $i$ let $f_{\beta_i,\beta_{i+1}}$, $g_{\beta_{i-1},\beta_{i}}$ be the paired estimators corresponding to this schedule, and $\mu_i=\mathbb{E}[f_{\beta_i,\beta_{i+1}}]$, $\nu_i=\mathbb{E}[g_{\beta_{i-1},\beta_{i}}]$ . 
\paralgo\ estimates $Q$ by running \proc on each $\gibbschain{H}{\beta_i}$, to estimate $\mu_i$ and $\nu_i$s each with precision  $\varepsilon' = \nicefrac{(\sqrt[l]{1 + \varepsilon} - 1)}{(\sqrt[l]{1 + \varepsilon} + 1)}$. Note that by this setting, $Q$ will be approximated within multiplicative factor of $ \left( \nicefrac{1 + \varepsilon'}{1 - \varepsilon'} \right)^{\ell}$. Assume $\tau_i$ is the true relaxation time of $\gibbschain{H}{\beta_i}$ and
suppose $\EigTwoBound_i$ is a known upper bound on the second eigenvalue of $\gibbschain{H}{\beta_i}$, thus $(\EigTwoBound_i-1)^{-1}\log (2)\geq \tau_i$. The following hold and thm 2.7 is immediately concluded from it:

\begin{lemma}\label{lem:rangeparalgo}
Let  $H_{\max} \doteq \max_{x\in \Omega}H(x)$. we have:
\begin{enumerate}[wide, labelwidth=0pt, labelindent=8pt]\setlength{\itemsep}{2pt}\setlength{\parskip}{0pt}
\vspace{-0.2 cm}
    \item 
 for all $1\leq i\leq \ell$, ${\rm Range}(f_{\beta_i,\beta_{i+1}})/\mu_i\leq \ell^{1/\log (n)}$, 
 \item for all $1\leq i\leq \ell$, ${\rm Range}(g_{\beta_{i-1},\beta_i})/\nu_i\leq \ell^{\alpha_0(i)/\log n}$, where $\alpha_0(i)=(\nicefrac{ H_{\max}}{2\mathbb{E}[H(x)]})-1 , ~ x\sim{\gibbs{\beta_i}}$.
 \end{enumerate}
 \end{lemma}
\begin{proof}

Let $\Delta_i=\beta_{i+1}-\beta_i$. Thus, $f_i(x)=\exp\left(\frac{-\Delta_i}{2}H(x)\right)$ and $g_i(x)=\exp\left(\frac{\Delta_i}{2}H(x)\right)$. So we have:
\[
{\rm Range}(f_i)=\exp\left(\frac{-\Delta_i}{2}\min_x H(x)\right)-\exp\left(\frac{-\Delta_i}{2}\max_x H(x)\right)\leq \exp\left(\frac{-\Delta_i}{2}\min_x H(x)\right)
\]

and 
\[
{\rm Range}(g_i)=\exp\left(\frac{\Delta_i}{2}\max_x H(x)\right)-\exp\left(\frac{\Delta_i}{2}\min_x H(x)\right)\leq \exp\left(\frac{\Delta_i}{2}\max_x H(x)\right)
\]

\[
\mu_i = Z(\beta_{i}+\Delta_{i}/2)/Z(\beta_i)\quad \& \quad
\nu_i = Z( \beta_{i+1}-\Delta_{i}/2)/Z(\beta_{i+1}) 
\]

\begin{align}
    \frac{{\rm Range}(f_i)}{\mu_i} 
    & \leq  \frac{\exp\left(\frac{-\Delta_i}{2}\min_x H(x)\right)}{\exp\left(z(\beta_{i}+\Delta_i/2)-z(\beta_i)\right)}~,
    \frac{{\rm Range}(g_i)}{\nu_i} 
    & \leq  \frac{\exp\left(\frac{\Delta_i}{2}\max_x H(x)\right)}{\exp\left(z(\beta_{i+1}-\Delta_i/2)-z(\beta_{i+1})\right)}
\end{align}

Writing $\Delta_i/2= \frac{\Delta_i/2}{z(\beta_i+\Delta_i/2)-z(\beta_i)}(z(\beta_i+\Delta_i/2)-z(\beta_i))$, we get:
\begin{align*}
\frac{{\rm Range}(f_i)}{\mu_i}&\leq \exp\left(-\frac{\Delta_i}{2}\min_{x}H(x)-\left(z(\beta_{i}+\Delta_i/2)-z(\beta_i)\right)\right)
\\
&\leq \exp\left(\left(z(\beta_{i}+\Delta_i/2)-z(\beta_i)\right)\left(\frac{-\Delta_i\cdot \min_x H(x)}{2\left(z(\beta_{i}+\Delta_i/2)-z(\beta_{i})\right)}-1\right)\right)
\end{align*}
and 
$$\frac{{\rm Range}(g_i)}{\nu_i}\leq \exp\left(\left(z(\beta_{i+1}-\Delta_i/2)-z(\beta_{i+1})\right)\left(\frac{\Delta_i\cdot \max_x H(x)}{2\left(z(\beta_{i+1}-\Delta_i/2)-z(\beta_{i+1})\right)}-1\right)\right)$$

Let $z'$ and $z''$ be the first and second derivative of $z$ with respect to $\beta$. Note that $z'(\beta)=\mathbb{E}_{x\sim \gibbs{\beta}}[-H(x)]$. Since $z''\geq 0$ we have:   
$$   z'(\beta_i) < \frac{z(\beta_{i}+\Delta_i/2)-z(\beta_i)}{\Delta_i/2} < z'(\beta_{i}+\Delta_i/2)
$$
and 
$$ z'(\beta_{i+1}-\Delta_i/2) < \frac{z(\beta_{i+1})-z(\beta_{i+1}-\Delta_i/2)}{\Delta_i/2} < z'(\beta_{i+1}).$$
Which are equivalent to 
$
\frac{1}{z'(\beta_{i}+\Delta_i/2)} \leq \frac{\Delta_i/2}{z(\beta_{i}+\Delta_i/2)-z(\beta_i)} \leq \frac{1}{z'(\beta_{i})}
$ and $
\frac{1}{z'(\beta_{i+1})} \leq \frac{\Delta_i/2}{z(\beta_{i+1})-z(\beta_{i+1}-\Delta_i/2)} \leq \frac{1}{z'(\beta_{i+1}-\Delta_i/2)}
$.

Therefore, 
\begin{align}\label{eq:10}
    \frac{{\rm Range}(f_i)}{\mu_i} 
    & \leq \exp\left(\left(z(\beta_{i}+\Delta_i/2)-z(\beta_i)\right)\left(\frac{-\min_x H(x)}{2}\frac{1}{z'(\beta_{i})}-1\right)\right)\\
    & = \exp\left(\left(z(\beta_{i}+\Delta_i/2)-z(\beta_i)\right)\left(\frac{\min_x H(x)}{2} \frac{1}{\mathbb{E}[H]}-1\right)\right)\\
    &\leq \exp\left(z(\beta_i)-z(\beta_{i}+\Delta_i/2)\right)
\end{align}

Similarly for range of $g_i$s we have:

\begin{align}\label{eq:13}
    \frac{{\rm Range}(g_i)}{\nu_i} 
    & \leq \exp\left(\left(z(\beta_{i+1}-\Delta_i/2)-z(\beta_{i+1})\right)\left(\frac{-\max_x H(x)}{2}\frac{1}{z'(\beta_{i})}-1\right)\right)\\
    &\leq \exp\left(\left(z(\beta_{i+1}-\Delta_i/2)-z(\beta_{i+1})\right)\left(\frac{\max_x H(x)}{2\mathbb{E}_{\gibbs{\beta_i}}[H(X)]}-1\right)\right)
\end{align}

We now use \eqref{eq:10} together with  \cref{lem:intervalbd}. Setting $d=1$  we have, 
\begin{align*}
    \mathbb{P}\left(z{(\beta_{i})}-z{(\beta_{i+1})}>\frac{\log (3l/4)}{\log n}\right)=\exp(-\frac{\log (3l/4)}{\log n} \cdot k) &= (3/4)\exp(-\log l/\log n (\log n))\\
    &=(3/4)(1/l).
\end{align*}

Using union bound over all $1\leq i\leq \ell$ and that $z(\beta_i)-z(\beta_{i+1})\geq z(\beta_i)-z(\beta_{i}+\Delta_i/2)$, we conclude that with probability at least $3/4$ we have that for all $f_i$, ${\rm Range}(f_i)/\mu_i\leq \ell^{1/\log (n)}$.

Similarly using \eqref{eq:13}, the union bound, lemma \ref{lem:intervalbd} and that $z(\beta_i)-z(\beta_{i+1})\geq z(\beta_i-\Delta_i/2)-z(\beta_{i+1})$,
we can show that with constant probability all  $g_i$s generated by the \tpa\ schedule obey:  $\forall g_i;1\leq i\leq \ell, ~ {\rm Range}(g_i)/\nu_i\leq \exp\left((\log l/\log  n)\cdot (\alpha)\right)=\ell^{\alpha_0/\log n}$, where $\alpha_0=\frac{\max_x H(x)}{2\mathbb{E}_{\gibbs{\beta_i}}[H(X)]}-1 $.
\end{proof}

The following corollary is concluded from \cref{lem:rangeparalgo} and relative trace variance bounds:

\begin{coro} When $\varepsilon\leq \ell^{1/\log (n)}(1+\ell^{\alpha_0(i)}) \cdot\frac{\ell\tau_{\beta_i}}{(1-\Lambda_i)^{-1}}$,  \proc\ invoked on the $i$th iteration will stop using sample consumption of $\tilde{O}\left({\ell^2}\tau_i {\rm Reltrv}_i\right)$ note that this is improvement over classic bounds which are $\tilde{O}\left((1-\Lambda_i)^{-1} \mathbb{V}{\rm rel}_i\right)$. In total the sample complexity of \paralgo\ for  $\varepsilon\leq \ell^{1/\log (n)}\min_i(1+\ell^{\alpha_0(i)}) \cdot\frac{\ell\tau_{\beta_i}}{(1-\Lambda_i)^{-1}}$ is dominated by 
$ \tilde{O}\left({\ell^2}\sum_{i=1}^{\ell}\tau_i {\rm Reltrv}_i\right)~.$

\end{coro}

\subsection{Further experimental results}\label{app:ex}

\begin{figure*}[h]
\hspace{1cm}
\scalebox{1}{
  \begin{subfigure}{0.3\textwidth}
    \includegraphics[width=6cm,height=6cm]{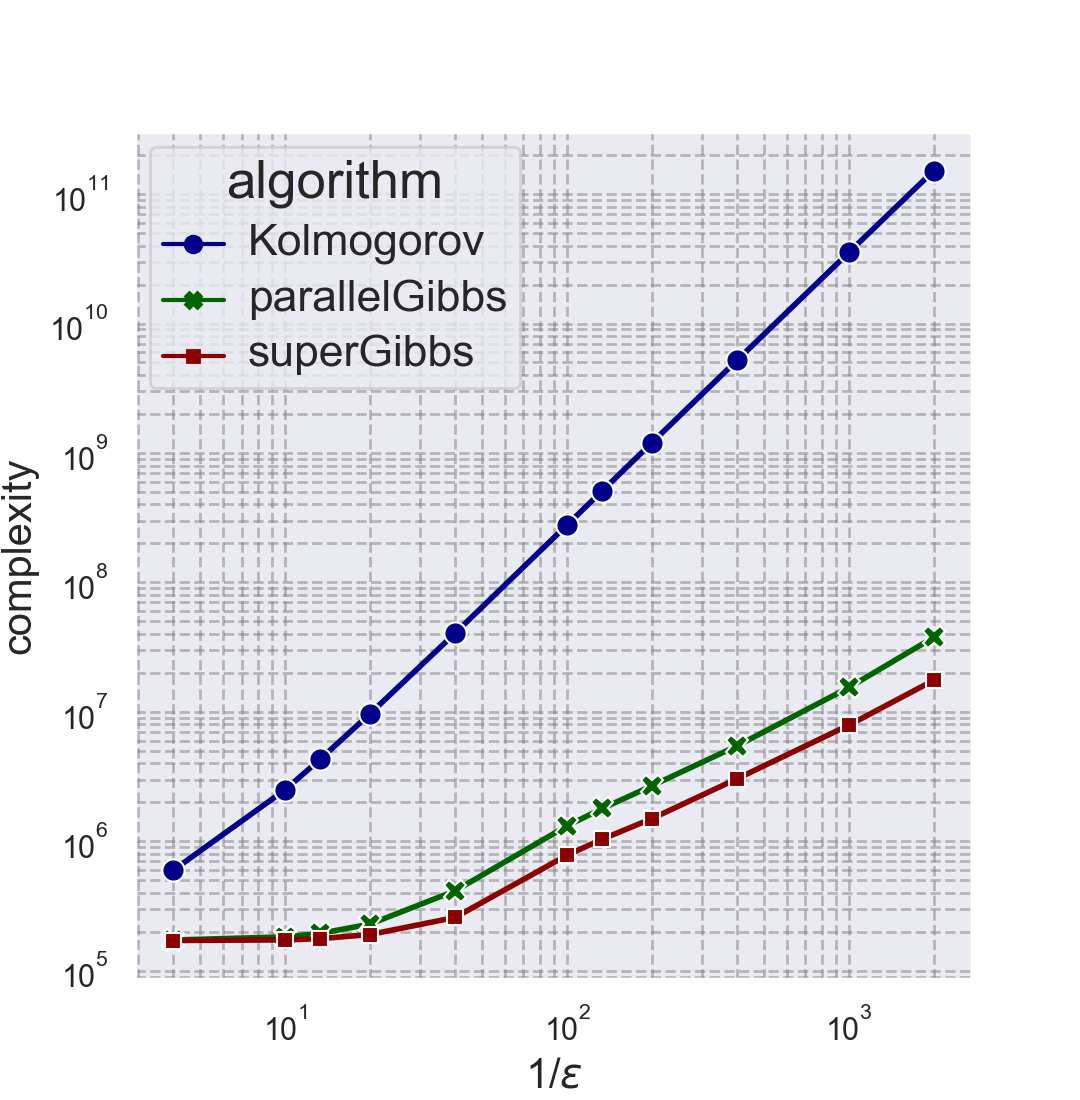}
    \caption{$\beta = .05$, 2$\times$2 lattice} \label{fig:1a}
   \end{subfigure}
  }%
  \hspace*{\fill}
  \scalebox{1}{
  \begin{subfigure}{0.3\textwidth}
    \includegraphics[width=6cm,height=6cm]{figures/complexity_vs_eps_Ising2_2_v1.png}
    \caption{$\beta = .01$, 3$\times$3 lattice} \label{fig:1b}
  \end{subfigure}
  }

\hspace{1cm}
  \scalebox{1}{
  \begin{subfigure}{0.3\textwidth}
    \includegraphics[width=6cm,height=6cm]{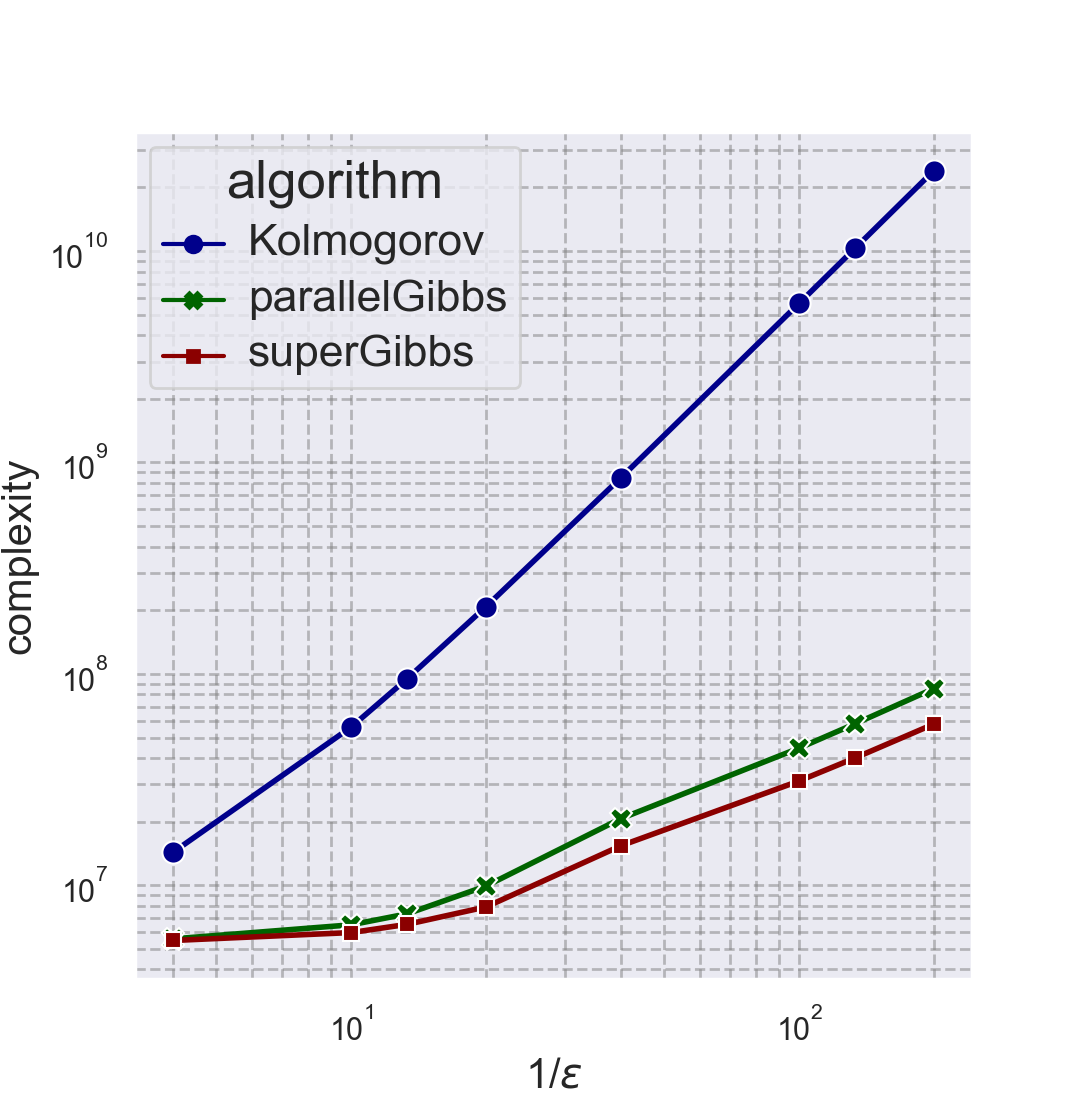}
    \caption{$\beta = .02$, 4$\times$4 lattice} \label{fig:1c}
  \end{subfigure}
  }%
  \hspace*{\fill}
  \scalebox{1}{
  \begin{subfigure}{0.3\textwidth}
    \includegraphics[width=6cm,height=6cm]{figures/complexity_vs_eps_Ising4_2_v1.png}
    \caption{$\beta = .002$, 6$\times$6 lattice} \label{fig:1d}
  \end{subfigure}}%
\caption{\footnotesize Comparison of sample complexity  on Ising models.}
\end{figure*}

 \bibliographystyle{abbrv}
\bibliography{ArXiv}
\end{document}